\documentclass{article}
\usepackage[utf8]{inputenc}
\usepackage[a4paper]{geometry}
\usepackage{amsmath,amsbsy,amsgen,amscd,amsthm,amsfonts,amssymb,mathrsfs}

\usepackage{natbib}
\setcitestyle{authoryear,round,comma}

\usepackage[dvipsnames]{xcolor}
\usepackage{xr-hyper}
\usepackage[colorlinks=true, linkcolor=blue, breaklinks=true]{hyperref}

\usepackage{graphicx}  
\usepackage{caption}
\usepackage{subcaption}
\graphicspath{ {./images/} }

\usepackage{pgfplots}

\usepackage{titlesec}
\titleformat*{\section}{\large\bfseries}
\titleformat*{\subsection}{\normalsize\bfseries}
\titleformat*{\subsubsection}{\small\bfseries}

\theoremstyle{plain}
\newtheorem{theorem}{Theorem}[section]
\newtheorem{lemma}[theorem]{Lemma}

\newtheorem{observation}[theorem]{Observation}

\theoremstyle{definition}

\newtheorem{remark}[theorem]{Remark}

\usetikzlibrary{shapes.misc}
\tikzset{cross/.style={cross out, draw=black, minimum size=2*(#1-\pgflinewidth), inner sep=0pt, outer sep=0pt},
cross/.default={1pt}}

\newcommand{\C}{\mathbb{C}}
\newcommand{\R}{\mathbb{R}}
\newcommand{\Z}{\mathbb{Z}}

\newcommand{\G}{\mathcal{G}}
\renewcommand{\L}{\mathcal{L}}
\newcommand{\I}{\mathcal{I}}

\renewcommand{\i}{\mathbf{i}}

\newcommand{\SO}[1]{\operatorname{SO}(#1)}
\newcommand{\Tr}{\operatorname{Tr}}

\newcommand{\XX}[3]{X_{#1 #2}^{(#3)}}

\renewcommand{\AE}{\mathcal{A}_E}
\newcommand{\AI}{\mathcal{A}_I}

\definecolor{maroon}{rgb}{0.5, 0.0, 0.0}

\title{Orientation Estimation of Cryo-EM Images for Molecules with Arbitrary Symmetry by Semidefinite Programming}
\author{Yunpeng Shi\footnote{Department of Mathematics, University of California, Davis}
, Amit Singer\footnote{Department of Mathematics and Program in Applied and Computational Mathematics, Princeton University}, 
Jinxin Wang\footnote{Department of Statistics, University of Chicago},
Ruiyi Yang\footnote{School of Mathematical Sciences, MOE-LSC, Institute of Natural Sciences,  Shanghai Jiao Tong University (yry-ins@sjtu.edu.cn)}
}
\date{}

\begin{document}

\maketitle

\begin{abstract}
This paper extends the non-unique games framework introduced in \cite{bandeira2020non} to address symmetric structures. 
Precisely, we consider the problem 
\begin{align}
    \underset{g_1,\ldots,g_n\in \G}{\operatorname{min}}\,\, \frac{1}{|S|}\sum_{i,j=1}^n \sum_{s\in S} f_{ij}(g_i s g_j^{-1}), 
\end{align}
where $\G$ is a compact group, $S\subset \G$ is a finite subgroup of symmetry elements, and $f_{ij}:\G\rightarrow \R$ is a loss function. 
A notable example is common lines based orientation estimation of cryo-EM images for molecules with point group symmetry $S$, where $\G$ is the rotation group $\SO{3}$. 
With a semidefinite programming relaxation, we turn the problem into a convex program that is solved with the alternating direction method of multipliers.
Unlike existing approaches which are tailored to specific symmetry groups, ours provides a unified computational framework for addressing all point group symmetries.
We demonstrate the applicability of the proposed algorithm on the cryo-EM orientation estimation problem on both simulated and experimental datasets. 
In particular, our algorithm is effective in generating high-quality ab initio models from only a few class averages computed from experimental images. 
\end{abstract}

\section{Introduction}
Cryogenic electron microscopy (cryo-EM) has emerged as a powerful technique for determining the three-dimensional structures of macromolecules at near atomic resolution \citep{wang2006cryo,frank2006three,nakane2020single,yip2020atomic}.  
Symmetry, a fundamental feature of nature, is pervasive in the structural organization of protein molecules. 
However, incorporating symmetry in cryo-EM data analysis poses significant mathematical and computational challenges, especially in the orientation estimation problem for 3-D ab initio modeling.
In particular, existing common-lines methods do not extend naturally to the symmetric case because the pairwise relative orientations are estimated up to arbitrary symmetry group elements, which in turn compromise the subsequent steps. 

In this paper, we tackle this problem by extending the non-unique games (NUG) framework \citep{bandeira2020non} to handle symmetric structures.
Our proposed computational framework provides a unified treatment for all symmetry groups rather than a case-by-case design as done in existing works. 
More broadly, our framework addresses optimization problems over general compact groups from information about group offsets, covering many other important applications that include group synchronization and multi-reference alignment when the underlying signal has an intrinsic symmetric structure.

\subsection{The NUG Framework}

Suppose $\phi:\R^3 \rightarrow \R$ represents the electrostatic potential of a molecule with a non-trivial point group symmetry $S$ (a discrete subgroup of $\SO{3}$ with $|S|$ elements). 
That is,  $s\circ \phi = \phi$ for all $s\in S$, where the group action is defined as usual $s\circ \phi (x) = \phi(s^{-1}x)$. 
Consider the usual cryo-EM image formation model 
\begin{align}\label{eq:cryoem image formation}
    I_i = h_i\ast P(g_i \circ \phi) + \epsilon_i,\qquad i=1,\ldots,n,
\end{align}
where $g_i$ is the unknown rotation of the molecule corresponding to the $i$-th image, $P$ is the tomographic projection operator in the $z$-direction defined for a generic potential $\phi$ by 
\begin{align*}
    P(\phi)(x,y) = \int_\R \phi(x,y,z) dz,
\end{align*}
and $h_i$ is the inverse Fourier transform of the contrast transfer function (CTF) of the $i$-th image and assumed to be known, with $\ast$ denoting convolution. 
For simplicity, we are ignoring the effect of imperfect centering so that shifts are not included in the model \eqref{eq:cryoem image formation}.

The study of the orientation estimation problem dates back to the angular reconstitution method for non-symmetric molecules proposed by \cite{vainshtein1986determination,van1987angular}.
The crucial observation is that common lines between three images uniquely determine their relative orientations up to handedness. 
The method proceeds in a sequential manner by estimating the relative orientation of a new image with respect to already estimated ones based on the common lines information. 
Later, \cite{farrow1992posteriori} proposed obtaining the relative orientation of a new image in a least squares fashion.
A major issue with these sequential methods are their reliance on correct detection of common lines, which is challenging in low signal-to-noise regimes, and the accumulation of the estimation error. 
\cite{penczek1996common} tried to circumvent this issue by estimating the orientations for all images simultaneously by minimizing a global energy functional. 
However, this method is not very practical due to its overly large search space. 
\cite{mallick2006structure} introduced a Bayesian approach where the common lines between a pair of images are determined from their common lines with different image triplets.
An issue of this approach is the requirement of correct identification of at least seven common lines simultaneously, rendering it not suitable when the detection rate of correct common lines is low. 
An improved Bayesian approach was proposed by \cite{singer2010detecting} based on voting, which lowered the requirement to two common lines. 
Further advancements were introduced in \cite{singer2011three,shkolnisky2012viewing} based on semidefinite programming and synchronization approaches, leading to algorithms that are much more robust to noise.

The common-lines methods have made significant progress in solving the orientation estimation problem and demonstrated their empirical success on experimental datasets. 
Nevertheless, there is still further room for improvement in that these methods do not fully exploit all the information contained in the images. 
In particular, the methods mentioned above all take a two-step procedure by first computing the common lines (or a couple of top ones) and then estimating orientations that are most consistent with them.
This means that all information other than those contained in the detected common lines are discarded when estimating the orientations. 
At low signal-to-noise ratios, the true common lines may not appear among the top ones detected, making them no longer available in the second step.  

Motivated by this observation, \cite{bandeira2020non} proposed the NUG framework that estimates the orientations without explicit detection of common lines. 
Precisely, the authors consider the following optimization problem
\begin{align}
    \underset{g_1,\ldots,g_n\in \G}{\operatorname{min}}\,\, \sum_{i,j=1}^n f_{ij}(g_ig_j^{-1}), \label{eq:NUG}
\end{align}
where $\G=\SO{3}$ and $f_{ij}:\G\rightarrow\R$ is a measure of the agreement between the common lines in images $I_i$ and $I_j$ if their relative orientation was $g_ig_j^{-1}$ (see Section \ref{sec:cryoem fij} for its precise form).
On a high level, the NUG framework searches for the orientations that  maximize all pairwise agreement of common lines that would be implied by these orientations. 
The algorithm only exploits the common lines implicitly and leverages full information contained in all images. 
\cite{bandeira2020non} exploits representation theory to first turn \eqref{eq:NUG} into a matrix optimization problem with a linear objective 
\begin{align}\label{eq:NUG linear obj}
    \underset{\text{constraints on } X}{\operatorname{min}}\,\, \langle C,X\rangle 
\end{align}
for some matrix $C$ and $X=X(g_1,\ldots,g_n)$, followed by a semidefinite programming (SDP) relaxation on the constraints. 
The authors further demonstrated that the resulting algorithm is able to tolerate higher noise levels than other common-lines methods. 

It is worth mentioning that \eqref{eq:NUG} provides a general framework for optimization over compact groups $\G$ from information about group offsets. 
The terminology \emph{non-unique games} is motivated by the unique games problem in theoretical computer science \citep{khot2002power,khot2010unique}, which aims to solve a linear system over $\Z_L$ consisting of equations of the form $x_i-x_j \equiv b_{ij}$ with $\equiv$ denoting equality modulo $L$, for a set of pairs $(i,j)\in T$. 
Once $x_i$ is fixed, so is $x_j$ for any $j$ such that $(i,j)\in T$, hence the name unique games. 
This problem can be formulated as in \eqref{eq:NUG} by setting $\G=\Z_L$ and $f_{ij}(g)= -\delta_{g\equiv b_{ij}}$ as a Dirac delta function for $(i,j)\in T$ and $f_{ij}=0$ otherwise. 
For more general $\G$'s and $f_{ij}$'s, the original uniqueness property does not necessarily hold, leading to the non-unique games framework. 
Other important problems that can be cast in the form of \eqref{eq:NUG} include the semidefinite relaxation for the Max-Cut problem over $\Z_2$ \citep{goemans1995improved}, angular synchronization over $\SO{2}$ \citep{singer2011angular}, multi-reference alignment over $\SO{2}$ and $\SO{3}$ \citep{bandeira2014multireference,bandeira2020non}, the little Grothendieck problem over the orthogonal group $\mathrm{O}(d)$ \citep{bandeira2016approximating,naor2013efficient}, and many more.

\subsection{NUG with Symmetry}
A major issue with the original NUG framework is that when symmetry is present, all $|S|$ rotations along the orbit $\{g_is\}_{s\in S}$ result in the same image $I_i$, and hence the optimization problem \eqref{eq:NUG linear obj} has multiple global minima.
Linearity of \eqref{eq:NUG linear obj} then implies that any convex combination of the solutions is again a solution, so that it is hopeless to recover the orientations $g_i$'s from such an arbitrary solution $X$ to \eqref{eq:NUG linear obj}. 
For this reason, it is imperative to extend the NUG framework to address symmetric structures.

To derive the appropriate objective function for optimization, we notice that each pair of images now share $|S|$ lines in common due to symmetry. 
Therefore, we can regard the $i$-th image as $|S|$ slices, and the $j$-th image as $|S|$ slices, so each pair $(i,j)$ contributes $|S|^2$ terms to the objective function of the form 
\begin{align*}
    \sum_{s,s'\in S} f_{ij}\big((g_is)(g_js')^{-1}\big) =\sum_{s,s'\in S} f_{ij}\big(g_i(s{s'}^{-1}) g_j^{-1}\big) = |S|\sum_{s\in S} f_{ij}(g_isg_j^{-1}),
\end{align*}
where we used that $S$ is a subgroup.  
This then leads after a normalization to our NUG problem with symmetry 
\begin{align}\label{eq:NUG SO3}
    \underset{g_1,\ldots,g_n\in \G}{\operatorname{min}}\,\, \frac{1}{|S|}\sum_{i,j=1}^n \sum_{s\in S} f_{ij}(g_i s g_j^{-1}),
\end{align}
with $\G=\SO{3}$. 
It can be seen that the objective function \eqref{eq:NUG SO3} is constant along orbits and we can only hope to recover orbits, not individual rotations. 
We remark that the formulation is indeed invariant to a global rotation and reflection (see Section \ref{sec:invariance}), which may not be obvious at first glance. 
Again, the formulation \eqref{eq:NUG SO3} extends beyond $\SO{3}$ to general compact groups $\G$.

\subsection{Contributions}
In this paper we introduce an SDP solution to \eqref{eq:NUG SO3} for general compact groups $\G$ by extending the framework proposed in \cite{bandeira2020non}. 
The new SDP relaxation is able to give a better computational complexity than the original one as symmetry can annihilate certain degrees in the irreducible representations of the group elements, leading to an SDP of smaller size.
We focus on the orientation estimation problem for symmetric molecules in cryo-EM, although our framework also covers other applications such as multi-reference alignment with symmetry. 
The resulting algorithm returns accurate estimates of the orientations when moderate noise is present, and is shown to be effective in generating high-quality ab initio models from only a few class averages computed from experimental images of symmetric molecules.

\subsection{Related Works}
Few papers in the literature rigorously address the orientation estimation problem in the presence of symmetry.
Here we mention the line of works \cite{pragier2019common,rosen2020common,geva2023common} which perform case-by-case studies to different symmetry groups covering cyclic, dihedral, octahedral and tetrahedral, and are inspired by \cite{singer2010detecting,singer2011viewing,singer2011three,shkolnisky2012viewing} in the non-symmetric case. 
Our work presents a unified mathematical framework that simultaneously covers all the discrete symmetries in $\SO{3}$.

The rest of the paper is organized as follows. 
In Section \ref{sec:SDP}, we introduce a semidefinite programming solution to \eqref{eq: NUG symmetry}. 
In Section \ref{sec:euler angle recovery}, we discuss how to recover the Euler angles from the SDP solution.
Section \ref{sec:implementation} discusses practical considerations for implementation, followed by numerical experiments on both simulated and real datasets in Section \ref{sec:numerics}.

\section{An SDP Formulation}\label{sec:SDP}
In this section we shall introduce an SDP framework for approximating the solution to \eqref{eq:NUG SO3}, repeated below in a more formal way for general compact groups.  
Let $\G$ be a compact group and $S\subset \G$ a finite subgroup of symmetry elements. 
Consider the optimization problem
\begin{align}
    \underset{g_1,\ldots,g_n\in \G}{\operatorname{min}}\,\, \frac{1}{|S|}\sum_{i,j=1}^n \sum_{s\in S} f_{ij}(g_i s g_j^{-1}),  \label{eq: NUG symmetry}
\end{align}
where $f_{ij}:\G\rightarrow \R$ is a loss function. 
The road map is to exploit representation theory to first turn \eqref{eq: NUG symmetry} into a matrix optimization, followed by an SDP relaxation that yields a convex program summarized in \eqref{eq:general sdp framework}.

\subsection{Linearization via Fourier Expansion}
To start with, we shall first linearize the objective function via Fourier expansion over the compact group $\G$. 
Let $\{\rho_k\}_{k=1}^\infty$ be unitary irreducible representations of $\G$ whose dimensions are $d_k$'s \citep{fulton2013representation}, i.e., each $\rho_k:\G\rightarrow U(d_k)$ is a homomorphism into the space of unitary matrices of dimension $d_k$. 
We can expand each $f_{ij}$ in terms of a Fourier series
\begin{align}
    \frac{1}{|S|}\sum_{s\in S}f_{ij}(g_i s g_j^{-1})
    & = \frac{1}{|S|}\sum_{s\in S} \sum_{k=0}^\infty d_k \Tr\big[\widehat{f}_{ij}(k)\rho_k (g_i s g_j^{-1})\big] \nonumber \\
    & =\frac{1}{|S|} \sum_{s\in S} \sum_{k=0}^\infty d_k \Tr\big[\widehat{f}_{ij}(k) \rho_k (g_i)\rho_k(s)\rho_k(g_j)^*\big] \nonumber \\
    & =\sum_{k=0}^\infty \Tr\Big(d_k\widehat{f}_{ij}(k) \rho_k(g_i)\Big[\frac{1}{|S|}\sum_{s\in S}\rho_k(s)\Big] \rho_k(g_j)^*\Big) \label{eq: linear-FE},
\end{align}
where  
\begin{align*}
    \widehat{f}_{ij}(k) = \int_\G f_{ij}(g) \rho_k(g)^*dg
\end{align*}
is the Fourier coefficient on $\G$ with $dg$ the Haar measure on $\G$.
We note that the matrix 
\begin{align}\label{eq:general framework Ak}
    A_k:=\frac{1}{|S|} \sum_{s\in S} \rho_k(s) 
\end{align}
plays a crucial role in our proposed methodology and enjoys some nice properties as shown in the next result that will lead to our SDP framework.
\begin{lemma}\label{lemma:properties of Ak}
The matrix $A_k$ is a projection matrix, i.e., $A_k^*=A_k$ and $A_k^2=A_k$. 
\end{lemma}
\begin{proof}
Note that 
\begin{align*}
    A_k^* = \frac{1}{|S|}\sum_{s\in S} \rho_k(s)^*  = \frac{1}{|S|}\sum_{s\in S} \rho_k(s^{-1}) = \frac{1}{|S|}\sum_{s\in S} \rho_k(s) = A_k,
\end{align*}
and 
\begin{align*}
    A_k^2 = \frac{1}{|S|^2} \sum_{s\in S}\sum_{t\in S} \rho_k(s)\rho_k(t)=\frac{1}{|S|^2} \sum_{s\in S}\sum_{t\in S} \rho_k(st) =  \frac{1}{|S|}\sum_{s\in S} \rho_k(s) = A_k. 
\end{align*}
\end{proof}

In the non-symmetric case where $S=\{e\}$ is the trivial group, we see that $A_k$ equals the identity matrix. 
However, for a nontrivial symmetry group $S$, the $A_k$'s are in general rank deficient. 
This constitutes as the major technical hurdle in formulating an SDP relaxation to the linearized objective \eqref{eq: linear-FE}, which through the definition of $A_k$ can be recast to
\begin{align*}
    \frac{1}{|S|}\sum_{s\in S}f_{ij}(g_isg_j^{-1}) = \sum_{k=0}^\infty \Tr \Big( d_k\widehat{f}_{ij}(k) \big[\rho_k(g_i) A_k\big] \big[A_k^* \rho_k(g_j)^*\big]\Big).
\end{align*}
By introducing the variables 
\begin{align} \label{eq: Ck and X^k}
    C^{(k)} = d_k
    \begin{bmatrix}
        \widehat{f}_{11}(k) & \widehat{f}_{21}(k) & \cdots & \widehat{f}_{n1}(k) \\
        \widehat{f}_{12}(k) & \widehat{f}_{22}(k) & \cdots & \widehat{f}_{n2}(k) \\
        \vdots & \vdots & \ddots & \vdots \\
        \widehat{f}_{1n}(k) & \widehat{f}_{2n}(k) & \cdots & \widehat{f}_{nn}(k)
    \end{bmatrix},\qquad X^{(k)} = 
    \begin{bmatrix}
        \rho_k(g_1)A_k \\
        \vdots \\
        \rho_k(g_n)A_k
    \end{bmatrix}
    \begin{bmatrix}
        \rho_k(g_1)A_k \\
        \vdots \\
        \rho_k(g_n)A_k
    \end{bmatrix}^*,
\end{align}
we can write the objective function in \eqref{eq: NUG symmetry} as 
\begin{align}
    \frac{1}{|S|}\sum_{i,j=1}^n\sum_{s\in S} f_{ij}(g_isg_j^{-1}) &= \sum_{k=0}^\infty \sum_{i,j=1}^n \Tr\big[d_k\widehat{f}_{ij}(k) X^{(k)}_{ij}\big] \nonumber\\
    & = \sum_{k=0}^\infty \Tr \Big[\sum_{i,j=1}^n d_k\widehat{f}_{ij}(k)X_{ij}^{(k)} \Big] = \sum_{k=0}^\infty \Tr \Big( \sum_{i=1}^n [C^{(k)}X^{(k)}]_{ii}\Big), \label{eq:general framework trace}
\end{align}
where $[C^{(k)}X^{(k)}]_{ii}$ is the $ii$-th block of $C^{(k)}X^{(k)}$.
Now using the fact that the trace of a block matrix is the sum of traces of its diagonal blocks, we can further write 
\begin{align*}
    \frac{1}{|S|}\sum_{i,j=1}^n\sum_{s\in S} f_{ij}(g_isg_j^{-1}) = \sum_{k=0}^\infty  \sum_{i=1}^n\Tr\big([C^{(k)}X^{(k)}]_{ii}\big) = \sum_{k=0}^\infty \Tr\big(C^{(k)}X^{(k)}\big).
\end{align*}
Thus, we arrive at the following optimization problem
\begin{equation}\label{eq:general framework X=rho constraint}
\begin{aligned}
    \underset{X^{(k)}}{\operatorname{min}}\quad  &\sum_{k=0}^\infty \Tr\big(C^{(k)}X^{(k)}\big)\\
    \text{subject to } \quad 
    & X^{(k)}_{ij}= \rho_k(g_i) A_k \rho_k(g_j)^*. 
\end{aligned}
\end{equation}

\subsection{Convex Relaxation}

The optimization problem \eqref{eq:general framework X=rho constraint} now has a linear objective function, but the constraints are still highly nonlinear and nonconvex. 
In this subsection, we shall consider convex relaxations of \eqref{eq:general framework X=rho constraint}.

\subsubsection{Constraints on Diagonal Blocks}
Notice that for the diagonal blocks, we have $X_{ii}^{(k)}=\rho_k(g_i)A_k \rho_k(g_i)^*$ so that each $X_{ii}^{(k)}$ is unitarily equivalent to $A_k$. 
Since $A_k$ is a projection matrix by Lemma \ref{lemma:properties of Ak}, we shall impose the following constraints 
\begin{equation}\label{eq:diagonal constraints}
\begin{aligned}
    0  \preceq \,&X_{ii}^{(k)}  \preceq I_{d_k}\\
    \Tr [X_{ii}^{(k)}] = &\Tr[A_k]=\operatorname{rank}(A_k)=:r_k.
\end{aligned}
\end{equation}
We mention that this defines the so-called rank $r_k$ \emph{Fantope} \citep{dattorro2005convex}, namely the convex hull of all rank-$r_k$ orthogonal projectors, which is closely related to classical spectral relaxations.

\subsubsection{Constraints Coupling Different Degrees}
To couple the matrices for different $k$'s, we shall exploit the Dirac delta function on the identity element of $\G$, denoted by $\delta$.
More precisely, we have the following Fourier series expansion 
\begin{align}\label{eq:general framework Fourier expansion for delta}
    \delta(g)=\sum_{k=0}^\infty d_k\Tr\big[\widehat{\delta}(k)\rho_k(g)\big] = \sum_{k=0}^\infty d_k\Tr\Big[\Big(\int_\G \delta(h)\rho_k(h)^*\,dh\Big)\rho_k(g)\Big]=\sum_{k=0}^\infty d_k \Tr\big[\rho_k(g)\big]. 
\end{align}
Replacing $g$ by $g^{-1}(g_isg_j^{-1})$ and averaging over $S$, we obtain 
\begin{align}
    \frac{1}{|S|} \sum_{s\in S} \delta(g^{-1}g_isg_j^{-1}) &= \frac{1}{|S|} \sum_{s\in S}\sum_{k=0}^\infty d_k \Tr \big[\rho_k(g)^*\rho_k(g_i)\rho_k(s)\rho_k(g_j)^*\big]\nonumber\\
    &=\sum_{k=0}^\infty d_k\Tr \big[\rho_k(g)^* \rho_k(g_i)A_k \rho_k(g_j)^*\big]\overset{\eqref{eq: Ck and X^k}}{=}\sum_{k=0}^\infty d_k\Tr \big[\rho_k(g)^* X_{ij}^{(k)}\big], \quad \forall g\in\G.  \label{eq:delta function constraint}
\end{align}
Therefore, a natural convex relaxation for the above equation is that 
\begin{align*}
    \sum_{k=0}^\infty d_k\Tr \big[\rho_k(g)^* X_{ij}^{(k)}\big] \geq 0, \quad \forall g\in\G,\\
    \int_\G \Big(\sum_{k=0}^\infty d_k\Tr \big[\rho_k(g)^* X_{ij}^{(k)}\big]\Big) dg=1.
\end{align*}
Noticing that $\int_\G \rho_k(g)dg=0$ for any nontrivial irreducible representation, implying that 
\begin{align*}
    \int_\G \Tr \big[\rho_k(g)^* X_{ij}^{(k)}\big]dg= \Tr\Big[\Big(\int_\G \rho_k(g)dg \Big)^*X_{ij}^{(k)}\Big] = 0, \quad \forall k\neq 0.
\end{align*}
Therefore the second constraint above can be simplified as (by noting $d_0=1$)
\begin{align}\label{eq:Xij0=1}
    X_{ij}^{(0)}=1, \quad \forall i,j
\end{align}
so that we arrive at the relaxation 
\begin{equation}\label{eq:inequality constraints}
\begin{aligned}
1+\sum_{k=1}^\infty d_k\Tr \big[\rho_k(g)^* X_{ij}^{(k)}\big] &\geq 0, \quad \forall g\in\G. 
\end{aligned}
\end{equation}

\subsubsection{Constraints on $X^{(1)}$} \label{subsec: so3convexhull}
When the group $\G=\SO{3}$, which is the major interest in cryo-EM, there is an additional structure of $X^{(1)}$ that we can further exploit.  
Note that for $g\in \SO{3}$ represented in terms of  ZYZ-Euler angles $(\alpha,\beta,\gamma)$, we have $\rho_1(g)=W^{(1)}(\alpha,\beta,\gamma)$ and the following identity relating Wigner-D matrices of order 1 and rotation matrices:
\begin{align*}
    M^*U^*W^{(1)}(\alpha,\beta,\gamma)UM=R,
\end{align*}
where 
\begin{align*}
    M=\begin{bmatrix}
        0 & -1 & 0 \\
        0 & 0 & 1  \\
        1 & 0 & 0
    \end{bmatrix},\qquad 
    U=\begin{bmatrix}
    -\frac{\i}{\sqrt{2}} & 0 & \frac{1}{\sqrt{2}} \\
    0 & 1 & 0 \\
    -\frac{\i}{\sqrt{2}} & 0 & -\frac{1}{\sqrt{2}}
    \end{bmatrix},
\end{align*}
and $R$ is the rotation matrix corresponding to the ZYZ-Euler angles $(\alpha,\beta,\gamma)$.
Therefore, we have 
\begin{align*}
    M^*U^*X_{ij}^{(1)} UM =  M^*U^*\left[\rho_1(g_i)A_1\rho_1(g_j)\right]UM=\frac{1}{|S|}\sum_{s\in S}M^*U^* \rho_1(g_isg_j)UM \in \operatorname{conv}(\SO{3}),
\end{align*}
the convex hull of $\SO{3}$. 
Theorem 1.3 of \cite{saunderson2015semidefinite} suggests relaxing $R\in \operatorname{conv}(\SO{3})$ by the PSD condition
\begin{align*}
    Q(R):=\begin{bmatrix}
    \begin{smallmatrix}
    1-R_{11}-R_{22}+R_{33} & R_{13}+R_{31} & R_{12}-R_{21} & R_{23}+R_{32} \\
    R_{13}+R_{31} & 1+R_{11}-R_{22}-R_{33} & R_{23}-R_{32} & R_{12}+R_{21} \\
    R_{12}-R_{21} & R_{23}-R_{32} & 1+R_{11}+R_{22}+R_{33} & R_{31}-R_{13} \\
    R_{23}+R_{32} & R_{12}+R_{21} & R_{31}-R_{13} & 1-R_{11}+R_{22}-R_{33}
    \end{smallmatrix}
    \end{bmatrix}
    \succeq 0,
\end{align*}
leading to the relaxation 
\begin{align}\label{eq:quaternion constraint}
    Q\left(M^*U^*X_{ij}^{(1)} UM\right)\succeq 0,\qquad \forall i,j=1,\ldots,n. 
\end{align}
Combining \eqref{eq:diagonal constraints}, \eqref{eq:inequality constraints}, and \eqref{eq:quaternion constraint}, we arrive at our SDP relaxation of \eqref{eq:general framework X=rho constraint}:
\begin{equation}\label{eq:general sdp framework}
\begin{aligned}
    \underset{X^{(k)}}{\operatorname{min}}\quad  &\sum_{k=1}^\infty\Tr\big(C^{(k)}X^{(k)}\big)\\
    \text{subject to } \quad 
    & X^{(k)} \succeq 0, \quad \forall k\\
    & 0  \preceq X_{ii}^{(k)}  \preceq I_{d_k},\quad \forall i,k\\
    & \Tr [X_{ii}^{(k)}] = \Tr[A_k]=\operatorname{rank}(A_k),\quad \forall i,k\\
    & 1+\sum_{k=1}^\infty d_k\Tr \big[\rho_k(g)^* X_{ij}^{(k)}\big] \geq 0, \quad \forall g\in\G,\quad\forall i,j \\
    & Q\left(M^*T^*X_{ij}^{(1)} TM\right)\succeq 0,\qquad \forall i,j.
\end{aligned}
\end{equation}

\subsubsection{Constraints from $J$-Ambiguity}
For cryo-EM applications, there is an additional degree of freedom coming from the handedness ambiguity (a.k.a $J$-ambiguity). 
Precisely, an image with orientation $R$ from a molecule $\phi$ is indistinguishable from the image of orientation $JRJ^*$ from the reflected molecule $J\phi$, where $J$ is reflection across the $z$-axis, i.e.,
\begin{align}\label{eq:reflection matrix}
    J=\begin{bmatrix}
        1 & 0  & 0 \\
        0  & 1 & 0 \\
        0  & 0  & -1
    \end{bmatrix}
\end{align}
Notice that the symmetry group of the reflected molecule is now $JSJ^*$.
This suggests that there is an additional solution to the SDP \eqref{eq:general framework X=rho constraint} that takes the form 
\begin{align*}
    X_{ij}^{J,(k)} &= \rho_k(Jg_iJ^*) \left[\frac{1}{|S|}\sum_{s\in JSJ^*} \rho_k(s)\right] \rho_k(Jg_jJ^*)^* \\
    &= J_k \rho_k(g_i)\left[\frac{1}{|S|}\sum_{s\in S}\rho_k(s)\right] \rho_k(g_j)^*J_k = J_k \rho_k(g_i)A_k\rho_k(g_j)^*J_k,    
\end{align*}
where 
\begin{align*}
    J_k=
    \begin{bmatrix}
    \ddots &    &    &     &   \\
           & -1 &    &     &   \\
           &    & 1  &     &   \\
           &    &    &  -1 &   \\
           &    &    &     &  \ddots
    \end{bmatrix} \in \R^{d_k\times d_k}
\end{align*}
is a diagonal matrix with alternating $\pm 1$ (with the center entry being 1) 
and we have used the fact that $\rho_k(JgJ^*) = J_k \rho_k(g)J_k$ for $g\in \SO{3}$. 
Moreover, any weighted linear combination
\begin{align}\label{eq:general framework X=rho constraint J}
    X_{ij}^{(k)}(t) = t \rho_k(g_i) A_k \rho_k(g_j)^* + (1-t) J_k\rho_k(g_i) A_k \rho_k(g_j)^* J_k,\qquad t\in[0,1]
\end{align}
is a solution to \eqref{eq:general framework X=rho constraint},
We shall then aim for the solution with $t=\frac12$ thanks to the zero patterns in $X_{ij}^{(k)}(\frac12)$. 
For instance, the matrices for $k=1$ and $k=2$ take the form
\begin{align*}
   X_{ij}^{(1)}(\frac12) =
   \begin{bmatrix}
       * & 0 & * \\
       0 & * & 0 \\
       * & 0 & *
   \end{bmatrix},\qquad 
   X_{ij}^{(2)}(\frac12) = 
   \begin{bmatrix}
       * & 0 & * & 0 & * \\
       0 & * & 0 & * & 0 \\
       * & 0 & * & 0 & * \\
       0 & * & 0 & * & 0 \\
       * & 0 & * & 0 & * \\
   \end{bmatrix}.
\end{align*}
On a high level, such zero pattern can be imposed by the condition $Z(X)=0$, where $Z$ is the operator that evaluates the corresponding entries of $X$. 
As in \cite[Section 5.3.1]{bandeira2020non}, such an equality constraint needs not to be imposed explicitly.
It can be achieved by optimizing in \eqref{eq:general sdp framework} a permuted version of $X$ instead, where the PSD constraints on the new $X^{(k)}$'s enforce such zero pattern.

\subsection{Proximal Refinement}\label{sec:proximal}
Observe that feasible points $X^{(k)}$ (and thus optimal solutions) of \eqref{eq:general framework X=rho constraint} are of low rank due to the expression of $X^{(k)}$ in \eqref{eq: Ck and X^k} and the fact that $A_k$ can be quite low-rank in practical applications. %
However, solving the convex relaxation problem in \eqref{eq:general sdp framework} does not necessarily yield low-rank $X^{(k)}$. Inspired by the idea of incorporating nonconvex constraint from \cite[Section 6.5]{diamond2016general}, we enhance the low-rank property of $X^{(k)}$ according to the following procedure. 

First, we shall solve the SDP in \eqref{eq:general sdp framework} to obtain a solution $\overline{X}_1$ containing the blocks $\overline{X}_1^{(k)}$, which are then projected to the set of low-rank matrices (say by requiring ${\rm rank}(\overline{X}_1^{(k)}) = 2\times {\rm rank}(A_k)$). 
The resulting low rank-projected solution is denoted as $\widehat{X}_1$ and consider solving 
\begin{align*}
    \underset{X^{(k)}}{\operatorname{min}}\,\, \langle C,X \rangle + \frac{\eta}{2} \|X - \widehat{X}_1 \|_F^2
\end{align*}
with the same set of constraints \eqref{eq:general sdp framework} on $X^{(k)}$. 
Intuitively, such a new optimization formulation would make the resulting solution $X^{(k)}$ closer to the low-rank matrices $\widehat{X}_1^{(k)}$.
To facilitate the computation using the SDP framework, we simply discard the quadratic term $\frac{\eta}{2} \|X \|_F^2$ following \cite{diamond2016general}. 
Therefore, we shall optimize the new SDP with objective function $\langle C,X \rangle - \eta  \langle X, \widehat{X}_1 \rangle$ to obtain a solution $\overline{X}_2$.
Such a procedure can be repeated with projection ranks decreasing for several times, which we call proximal refinement due to its similarity to proximal algorithms. 
Instead of assigning a uniform penalty $\eta$ for each degree $k$ in $X$, we can further impose a degree-dependent penalty such as $\sum_{k=1}^\infty \eta_k\langle X^{(k)},\widehat{X}_1^{(k)}\rangle$ with different $\eta_k$'s.

\begin{remark}
We remark that one can also perform a local gradient descent search on the objective function $\sum_{k=1}^\infty \Tr\big(C^{(k)}X^{(k)}\big)=:f(g_1,\ldots,g_n)$ by treating it as a function defined over $\SO{3}^n$. 
However, we observe empirically that such a refinement does not improve over the proximal refinement so we shall not pursue it further. 
\end{remark}

\section{Euler Angle Recovery}\label{sec:euler angle recovery}
In this section we discuss how to recover the Euler angles $\{(\alpha_i,\beta_i,\gamma_i)\}_{i=1}^n$ from the $X$ matrix obtained by solving the SDP \eqref{eq:general sdp framework}. 
One possible approach is by least squares, where one searches for the best pairs of Euler angles that would generate an $X$ that is closest to the SDP solution. 
However, this approach is sensitive to the errors in the SDP solution and here we present a robust approach for the case of cyclic and dihedral symmetries by exploiting the special structures of the $A_k$'s.

\subsection{Cyclic $C_m$ Symmetry}
In the case of cyclic $C_m$ symmetry, we can choose an axis of symmetry (the $z$-axis) so that $A_k$ is always a diagonal matrix with entries 0 and 1. 
Here are two observations that we will need.
\begin{observation}\label{obs:Cm-1}
For all $m\geq 2$, $A_1\in \R^{3\times 3}$ takes the form 
\begin{align*}
    \begin{bmatrix}
        0 & 0 & 0\\
        0 & 1 & 0\\
        0 & 0 & 0
    \end{bmatrix}.
\end{align*}
\end{observation}

\begin{observation}\label{obs:Cm-2}
For all $m\geq 2$, $A_m\in\R^{(2m+1)\times (2m+1)}$ takes the form
\begin{align*}
A_m=
    \begin{bmatrix}
        1 &  &  &  & \mathbf{0}\\
        & \mathbf{0}  &        & &\\
          &   &1        & &\\
          &   &        & \mathbf{0} &\\
        \mathbf{0}  &   &        &        & 1
    \end{bmatrix},
\end{align*}
i.e., the top left, the center, and the bottom right entries of $A_m$ equal to 1 and all the other entries are zero. 
\end{observation}
Based on these two observations, we shall recover all the $\alpha_i$'s and $\beta_i$'s using $X^{(1)}$ and recover the $\gamma_i$'s using $X^{(m)}$.
In this case, $X_{ij}^{(1)}$ defined in \eqref{eq:general framework X=rho constraint J} takes the form 
\begin{align*}
    X_{ij}^{(1)} = 
    \begin{bmatrix}
        \frac12 e^{\i \alpha_i -\i \alpha_j} \sin\beta_i\sin\beta_j & 0 &  -\frac12 e^{\i \alpha_i +\i \alpha_j} \sin\beta_i\sin\beta_j\\
        0 & \cos\beta_i\cos\beta_j & 0\\
        -\frac12 e^{-\i \alpha_i -\i \alpha_j} \sin\beta_i\sin\beta_j &0 & \frac12 e^{-\i \alpha_i +\i \alpha_j} \sin\beta_i\sin\beta_j
    \end{bmatrix}.
\end{align*}

\paragraph{Recovering $\beta_i$'s:}
Since $\beta_i\in [0,\pi]$, we have $\sin\beta_i\geq 0$ so that we can extract $\sin\beta_i \sin\beta_j=2|[X_{ij}^{(1)}]_{11}|$. 
Together with $\cos\beta_i\cos\beta_j = [X_{ij}^{(1)}]_{22}$, we can form two matrices
\begin{align*}
    [B_1]_{ij} =  \sin\beta_i \sin\beta_j \,\qquad [B_2]_{ij}=\cos\beta_i\cos\beta_j. 
\end{align*}
Both $B_1$ and $B_2$ have rank 1 and we can recover (up to a global $\pm$ sign) 
\begin{align*}
    b_1=[\sin\beta_1\ldots,\sin\beta_n]^T,\qquad b_2=[\cos\beta_1\ldots,\cos\beta_n]^T
\end{align*}
by computing the leading eigenvectors of $B_1$ and $B_2$.
Then each $\beta_i$ can be recovered via 
\begin{align*}
    \beta_i = \arctan \left(\frac{[b_1]_i}{[b_2]_i}\right).
\end{align*}
Due to the global sign ambiguity, either $\{\beta_i\}_{i=1}^n$ or $\{\pi-\beta_i\}_{i=1}^n$ recovers the true $\beta_i$'s. 

\paragraph{Recovering $\alpha_i$'s:}
By extracting $e^{\i \alpha_i -\i \alpha_j}$ from $[X_{ij}^{(1)}]_{11}$, we can form the matrix 
\begin{align*}
    A_{ij}=e^{\i \alpha_i -\i \alpha_j},
\end{align*}
where the standard angular synchronization approach \citep{singer2011angular} returns estimates $\{e^{\i(\alpha_i+\mu)}\}_{i=1}^n =:\{e^{\i\widetilde{\alpha}_i}\}_{i=1}^n$ for a global shift $\mu$. 
To estimate $\mu$, we form by extracting from $[X_{ij}^{(1)}]_{13}$ another matrix
\begin{align*}
    \widetilde{A}_{ij} = e^{\i \alpha_i +\i \alpha_j}.
\end{align*}
In particular, we should have 
\begin{align*}
    e^{-2\i \mu} 
    \begin{bmatrix}
        e^{\i \widetilde{\alpha}_1} \\
        \vdots \\
        e^{\i \widetilde{\alpha}_n}
    \end{bmatrix}
    \begin{bmatrix}
        e^{\i \widetilde{\alpha}_1} \\
        \vdots \\
        e^{\i \widetilde{\alpha}_n}
    \end{bmatrix}^*
    =\widetilde{A}.
\end{align*}
By solving this equation in the least squares sense
\begin{align*}
    \underset{c\in \C}{\operatorname{min}}\,\, \left\|c\begin{bmatrix}
        e^{\i \widetilde{\alpha}_1} \\
        \vdots \\
        e^{\i \widetilde{\alpha}_n}
    \end{bmatrix}
    \begin{bmatrix}
        e^{\i \widetilde{\alpha}_1} \\
        \vdots \\
        e^{\i \widetilde{\alpha}_n}
    \end{bmatrix}^*-\widetilde{A}\right\|_F
\end{align*}
(see Appendix \ref{sec:least squares type 1} for the analytic solution), we can recover the $\alpha_i$'s from $\sqrt{c}\{e^{\i \widetilde{\alpha}_i}\}_{i=1}^n$. 
Note that here we take any one of the two possible values for $\sqrt{c}$ because $c$ is complex, so we can only recover the $\alpha_i$'s up to $\pi$, i.e., we eventually obtain $\{\alpha_i\}_{i=1}^n$ or $\{\alpha_i+\pi\}_{i=1}^n$. 

\paragraph{Recovering $\gamma_i$'s:}
We shall exploit the structure of $X^{(m)}_{ij}$ to recover the $\gamma_i$'s because they do not show up until in $X^{(m)}$.
Recall that for a rotation $g\in \SO{3}$ corresponding to Euler angles $(\alpha,\beta,\gamma)$, its representation is given by the Wigner-D matrices \citep{Varshalovich1988} that can be decomposed as 
\begin{align*}
    \rho_k(g) = D^{(k)}(\alpha)
    w^{(k)}(\beta)
    D^{(k)}(\gamma), \qquad 
    D^{(k)}(\theta)=
    \begin{bmatrix}
        e^{\i k\theta} & & \\
        & \ddots & \\
        & & e^{-\i k \theta} 
    \end{bmatrix},
\end{align*}
where $w^{(k)}(\beta)$ is the Wigner small d-matrix (see Appendix \ref{sec:wigner d} for more details). 
Now for $t=\frac12$ in \eqref{eq:general framework X=rho constraint J}, we can write the SDP solution as 
\begin{align*}
    X_{ij}^{(m)} & = \frac12 \rho_m(g_i)A_m \rho_m(g_j)^* +\frac12 J_m \rho_m(g_i)A_m \rho_m(g_j)^*  J_m \\
    &= \frac12 [D^{(m)}(\alpha_i) w^{(m)}(\beta_i) D^{(m)}(\gamma_i)] A_m [D^{(m)}(\alpha_j) w^{(m)}(\beta_j) D^{(m)}(\gamma_j)]^*\\
    &\quad +  \frac12 J_m[D^{(m)}(\alpha_i) w^{(m)}(\beta_i) D^{(m)}(\gamma_i)] A_m [D^{(m)}(\alpha_j) w^{(m)}(\beta_j) D^{(m)}(\gamma_j)]^*J_m\\
    & =\frac12 D^{(m)}(\alpha_i) w^{(m)}(\beta_i) [D^{(m)}(\gamma_i)A_m D^{(m)}(\gamma_j)^*]w^{(m)}(\beta_j)^* D^{(m)}(\alpha_j)^* \\
    &\quad +\frac12 J_mD^{(m)}(\alpha_i) w^{(m)}(\beta_i) [D^{(m)}(\gamma_i)A_m D^{(m)}(\gamma_j)^*]w^{(m)}(\beta_j)^* D^{(m)}(\alpha_j)^*J_m\\
    & = \frac12 D^{(m)}(\alpha_i) w^{(m)}(\beta_i) [D^{(m)}(\gamma_i)A_m D^{(m)}(\gamma_j)^*]w^{(m)}(\beta_j)^* D^{(m)}(\alpha_j)^* \\
    &\quad +\frac12 D^{(m)}(\alpha_i)J_m w^{(m)}(\beta_i) [D^{(m)}(\gamma_i)A_m D^{(m)}(\gamma_j)^*]w^{(m)}(\beta_j)^*J_m D^{(m)}(\alpha_j)^*,
\end{align*}
where in the last step we have used the fact that $D^{(m)}(\alpha)$ commutes with $J_m$ since both are diagonal matrices. 
Therefore we have
\begin{equation}\label{eq:DXD}
\begin{aligned}
D^{(m)}(\alpha_i)^* X_{ij}^{(m)} D^{(m)}(\alpha_j)& = \frac12 w^{(m)}(\beta_i)  [D^{(m)}(\gamma_i)A_m D^{(m)}(\gamma_j)^*]w^{(m)}(\beta_j)^*\\
&\quad +\frac12 J_mw^{(m)}(\beta_i)  [D^{(m)}(\gamma_i)A_m D^{(m)}(\gamma_j)^*]w^{(m)}(\beta_j)^*J_m,
\end{aligned}
\end{equation}
where  
\begin{align*}
    D^{(m)}(\gamma_i)A_m D^{(m)}(\gamma_j)^*=
    \begin{bmatrix}
        e^{\i m (\gamma_i-\gamma_j)} & & & &\\
        & \ddots &  & &\\
        & & 1 & & \\
        & & & \ddots & \\
        & & & & e^{-\i m(\gamma_i-\gamma_j)}
    \end{bmatrix}
\end{align*}
is a matrix with only three nonzero entries. 
If we denote $(w_i)_{-m}$, $(w_i)_0$, $(w_i)_m$ as the 1st, $(m+1)$th, and $(2m+1)$th column of $w^{(m)}(\beta_i)$, then we can further write
\begin{align*}
    &\quad  D^{(m)}(\alpha_i)^* X_{ij}^{(m)} D^{(m)}(\alpha_j)\\
    & = \frac12 e^{\i m (\gamma_i-\gamma_j)} \left[(w_i)_{-m}(w_j)_{-m}^*+J_m(w_i)_{-m}(w_j)_{-m}^*J_m\right] + \frac12 \left[(w_i)_{0}(w_j)_{0}^*+J_m(w_i)_{0}(w_j)_{0}^*J_m\right]\\
    &\quad+ \frac12 e^{-\i m (\gamma_i-\gamma_j)} \left[(w_i)_m(w_j)_m^*+J_m(w_i)_m(w_j)_m^*J_m\right]\\
    &=: e^{\i m (\gamma_i-\gamma_j)} W_{-m}+ W_0 + e^{-\i m (\gamma_i-\gamma_j)} W_m,
\end{align*}
where $D^{(m)}(\alpha_i)^* X_{ij}^{(m)} D^{(m)}(\alpha_j)$, $W_{-m}$, $W_0$, $W_m$ are all known matrices. 
We then estimate $e^{\i m (\gamma_i-\gamma_j)}$ for each pair of $i,j$ from the least squares problem (see Appendix \ref{sec:least squares type 2} for details on the analytic solution)
\begin{align*}
    \underset{c\in \mathbb{C}}{\operatorname{min}} \,\, \|c W_{-m}+ W_0 + \bar{c}W_m-D^{(m)}(\alpha_i)^* X_{ij}^{(m)} D^{(m)}(\alpha_j)\|_F. 
\end{align*}
to obtain the matrix $G_{ij}=e^{\i m (\gamma_i-\gamma_j)}$.
Applying synchronization, we obtain $\{e^{\i (m\gamma_i+\gamma)}\}_{i=1}^n$ or equivalently $\{\gamma_i+\frac{\gamma}{m}+\frac{2k_i\pi}{m}\}_{i=1}^n$ for a global shift $\gamma$ and unknown $\{k_i\}_{i=1}^n$ arising from symmetry.

\subsection{Dihedral $D_m$ Symmetry $m\geq 3$}

Similarly as in the cyclic case, we have the following two observations
\begin{observation}\label{obs:Dm-1}
For all $m\geq 3$, $A_2\in \R^{5\times 5}$ takes the form 
\begin{align*}
    \begin{bmatrix}
        0 & 0 & 0 & 0 & 0\\
        0 & 0 & 0 & 0 & 0\\
        0 & 0 & 1 & 0 & 0\\
        0 & 0 & 0 & 0 & 0\\
        0 & 0 & 0 & 0 & 0
    \end{bmatrix}.
\end{align*}
\end{observation}

\begin{observation}\label{obs:Dm-2}
For all $m\geq 3$, $A_m\in \R^{(2m+1)\times (2m+1)}$ takes the form 
\begin{align*}
    \begin{bmatrix}
        \frac12 &  & \mathbf{0} &  & \frac12\\
         & \mathbf{0}  & & \mathbf{0} & \\
         \mathbf{0}&  & 0 &  & \mathbf{0}\\
         & \mathbf{0} &  & \mathbf{0} & \\
        \frac12 &  & \mathbf{0} &  & \frac12
    \end{bmatrix}
    \quad \text{when $m$ is odd, or} \quad 
    \begin{bmatrix}
        \frac12 &  & \mathbf{0} &  & \frac12\\
         & \mathbf{0}  & & \mathbf{0} & \\
         \mathbf{0}&  & 1 &  & \mathbf{0}\\
         & \mathbf{0} &  & \mathbf{0} & \\
        \frac12 &  & \mathbf{0} &  & \frac12
    \end{bmatrix}
    \text{when $m$ is even},
\end{align*}
i.e., all entries are zero except the four corner entries being $\frac12$ and the middle entry being 1 when $m$ is even. 
\end{observation}

Based on these two observations, we shall recover all the $\alpha_i$'s and $\beta_i$'s using $X^{(2)}$ and recover the $\gamma_i$'s using $X^{(m)}$. 
Notice that $X_{ij}^{(2)}$ takes the form 
\begin{align*}
    X_{ij}^{(2)}=
    \begin{bmatrix}
\begin{smallmatrix}
        \frac38 S_i^2S_j^2A_i^2\bar{A}_j^2 & 0 & \frac{\sqrt{6}}{8}S_i^2 (3C_j^2-1) A_i^2& 0 & \frac38 S_i^2S_j^2A_i^2A_j^2\\
        0 & \frac32 S_iC_iS_jC_jA_i\bar{A}_j & 0 & -\frac32 S_iC_iS_jC_jA_iA_j & 0\\
        \frac{\sqrt{6}}{8} (3C_i^2-1)S_i\bar{A}_j & 0 & \frac14(3C_i^2-1)(3C_j^2-1) & 0 & \frac{\sqrt{6}}{8}(3C_i^2-1)S_j^2A_j\\
        0 & -\frac32 S_iC_iS_jC_j \bar{A}_i \bar{A}_j & 0& \frac32 S_iC_iS_jC_j\bar{A}_iA_j & 0 \\
        \frac38 S_i^2S_j^2\bar{A}_i^2 \bar{A}_j^2 & 0 & \frac{\sqrt{6}}{8}S_i^2(3C_j^2-1)\bar{A}_i^2 & 0 & \frac38 S_i^2S_j^2 \bar{A}_i^2A_j^2
        \end{smallmatrix}
    \end{bmatrix},
\end{align*}
where $S_i=\sin\beta_i$, $C_i=\cos\beta_i$, and $A_i=e^{\i\alpha_i}$.

\paragraph{Recovering $\beta_i$'s:}
To recover the $\beta_i$'s, we shall exploit the entries
\begin{align*}
    \left[X_{ij}^{(2)}\right]_{11} = \frac38\sin^2\beta_i\sin^2\beta_je^{2\i(\alpha_i-\alpha_j)},\quad \left[X_{ij}^{(2)}\right]_{33} = \frac14(3\cos^2\beta_i-1)(3\cos^2\beta_j-1).
\end{align*}
In particular we have the following identity
\begin{align*}
    \cos^2\beta_i\cos^2\beta_j &= \frac23\left[\frac14(3\cos^2\beta_i-1)(3\cos^2\beta_j-1)-\frac34\sin^2\beta_i \sin^2\beta_j+\frac12\right]\\
    &= \frac23 \left[\left[X_{ij}^{(2)}\right]_{33} -2\left|\left[X_{ij}^{(2)}\right]_{11}\right|+\frac12\right].
\end{align*}
Forming the matrix
\begin{align*}
    [B]_{ij}= \cos^2\beta_i \cos^2\beta_j, 
\end{align*}
we can recover $[\cos^2\beta_1,\ldots,\cos^2\beta_n]^T$ as $\sqrt{\lambda_1}v_1$, where $(\lambda_1,v_1)$ is the leading eigenpair of $B$.  
Note that there is no sign ambiguity in this step since $[\cos^2\beta_1,\ldots,\cos^2\beta_n]^T$ is supposed to be nonnegative. 
This then allows us to recover each $\pm \cos\beta_i$ and hence $\beta_i$ or $\pi-\beta_i$. 
\begin{remark}
Unlike in the cyclic case where we can recover all $\{\beta_i\}_{i=1}^n$ or $\{\pi-\beta_i\}_{i=1}^n$, here we can only recover for each $i$ either $\beta_i$ or $\pi-\beta_i$. 
This is the additional inherent ambiguity in dihedral symmetry.    
\end{remark}

\paragraph{Recovering $\alpha_i$'s:}
After obtaining the estimate of $\beta_i$'s, we shall exploit the entries 
\begin{align*}
    \left[X_{ij}^{(2)}\right]_{11}=\frac32 \sin\beta_i\cos\beta_i\sin\beta_j\cos\beta_j e^{\i (\alpha_i+\alpha_j)} = \frac{3}{8}\sin(2\beta_i)\sin(2\beta_j)e^{\i (\alpha_i-\alpha_j)},\\
    \left[X_{ij}^{(2)}\right]_{13}=-\frac32 \sin\beta_i\cos\beta_i\sin\beta_j\cos\beta_j e^{\i (\alpha_i+\alpha_j)} = -\frac{3}{8}\sin(2\beta_i)\sin(2\beta_j)e^{\i (\alpha_i+\alpha_j)}.
\end{align*}
In particular, we have 
\begin{align*}
    \frac83\cdot\frac{\left[X_{ij}^{(2)}\right]_{11}}{\sin(2\beta_i)\sin(2\beta_j)}=e^{\i (\alpha_i-\alpha_j)},\quad 
    -\frac83\cdot\frac{\left[X_{ij}^{(2)}\right]_{13}}{\sin(2\beta_i)\sin(2\beta_j)}=e^{\i (\alpha_i+\alpha_j)}. 
\end{align*}
The rest follows as in the cyclic case. %

\paragraph{Recovering $\gamma_i$'s:}

As in \eqref{eq:DXD}, we have 
\begin{align*}
D^{(m)}(\alpha_i)^* X_{ij}^{(m)} D^{(m)}(\alpha_j)& = \frac12 w^{(m)}(\beta_i)  [D^{(m)}(\gamma_i)A_m D^{(m)}(\gamma_j)^*]w^{(m)}(\beta_j)^*\\
&\quad +\frac12 J_mw^{(m)}(\beta_i)  [D^{(m)}(\gamma_i)A_m D^{(m)}(\gamma_j)^*]w^{(m)}(\beta_j)^*J_m,
\end{align*}
where using Observation \ref{obs:Dm-2}
\begin{align*}
    D^{(m)}(\gamma_i)A_m D^{(m)}(\gamma_j)^*&=
    \begin{bmatrix}
        \frac12e^{\i m (\gamma_i-\gamma_j)} & & & & \frac12 e^{\i m (\gamma_i+\gamma_j)}\\
        & \ddots &  & &\\
        & & 0\text{ or } 1& & \\
        & & & \ddots & \\
        \frac12e^{-\i m(\gamma_i+\gamma_j)}& & & & \frac12e^{-\i m(\gamma_i-\gamma_j)}
    \end{bmatrix}
    \\
    &=:
    \begin{bmatrix}
        \frac12G^-_{ij} & & & & \frac12 G^+_{ij}\\
        & \ddots &  & &\\
        & & c_0& & \\
        & & & \ddots & \\
        \frac12\overline{G^+_{ij}}& & & & \frac12\overline{G^-_{ij}}
    \end{bmatrix}.
\end{align*}
If we denote $(w_i)_{-m}$, $(w_i)_0$, $(w_i)_m$ as the 1st, $(m+1)$th, and $(2m+1)$th column of $w^{(m)}(\beta_i)$, then we can further write
\begin{align*}
    & \quad D^{(m)}(\alpha_i)^* X_{ij}^{(m)} D^{(m)}(\alpha_j)\\
    & = \frac14 G^-_{ij}\left[(w_i)_{-m}(w_j)_{-m}^*+J_m(w_i)_{-m}(w_j)_{-m}^*J_m\right] + \frac14 \overline{G^+_{ij}}\left[(w_i)_{m}(w_j)_{-m}^*+J_m(w_i)_{m}(w_j)_{-m}^*J_m\right] \\
    &\quad + \frac14c_0\left[(w_i)_{0}(w_j)_{0}^*+J_m(w_i)_{0}(w_j)_{0}^*J_m\right]+ \frac14 G^+_{ij} \left[(w_i)_{-m}(w_j)_m^*+J_m(w_i)_{-m}(w_j)_m^*J_m\right]\\
    &\quad + \frac14 \overline{G^-_{ij}}\left[(w_i)_m(w_j)_m^*+J_m(w_i)_m(w_j)_m^*J_m\right]\\
    &=: G^-_{ij} W_{-m,-m}+ \overline{G^+_{ij}} W_{m,-m} + W_0 + G^+_{ij} W_{-m,m} +  \overline{G^-_{ij}} W_{m,m},
\end{align*}
where $D^{(m)}(\alpha_i)^* X_{ij}^{(m)} D^{(m)}(\alpha_j)$, $W_{-m,-m},W_{m,-m},W_0,W_{-m,m},W_{m,m}$ are all known matrices. 
We then estimate $G^-_{ij},G^+_{ij}$ for each pair of $i,j$ from the least squares problem (see Appendix \ref{sec:least squares type 3} for details on the analytic solution)
\begin{align*}
    \underset{G^-_{ij},G^+_{ij}\in \mathbb{C}}{\operatorname{min}} \,\, \left\|G^-_{ij} W_{-m,-m}+ \overline{G^+_{ij}} W_{m,-m} + W_0 + G^+_{ij} W_{-m,m} +  \overline{G^-_{ij}} W_{m,m}-D^{(m)}(\alpha_i)^* X_{ij}^{(m)} D^{(m)}(\alpha_j)\right\|_F. 
\end{align*}
to obtain the matrix $G^-_{ij}=e^{\i m (\gamma_i-\gamma_j)}$.
Applying synchronization, we obtain $\{e^{\i (m\gamma_i+\gamma)}\}_{i=1}^n$ or equivalently $\{\gamma_i+\frac{\gamma}{m}+\frac{2k_i\pi}{m}\}_{i=1}^n$ for a global shift $\gamma$ and unknown $\{k_i\}_{i=1}^n$ arising from symmetry.

\section{Practical Implementation}\label{sec:implementation}

In this section, we shall present details for practical implementation regarding several ingredients of our proposed approach. 

\subsection{Precise Form of $f_{ij}$ for Orientation Estimation}\label{sec:cryoem fij}
For the orientation estimation problem in cryo-EM, $f_{ij}$ is defined as follows. 
Let $\widehat{I}_i(r,\theta)$ and $\widehat{I}_i(r,\theta)$ be the Fourier transforms of the images $I_i$ and $I_j$ represented in polar coordinates. 
If the orientations for $I_i$ and $I_j$ are $g_i$ and $g_j$ respectively, then the polar angles $\alpha_{ij}$ and $\alpha_{ji}$ within $\widehat{I}_i(r,\theta)$ and $\widehat{I}_i(r,\theta)$ that correspond to the common line satisfy \citep{singer2011viewing}
\begin{align*}
    g_i^{-1}(\cos\alpha_{ij},\sin\alpha_{ij},0)^T =  g_j^{-1}(\cos\alpha_{ji},\sin\alpha_{ji},0)^T = \frac{(g_i^{-1}e_3)\times (g_j^{-1}e_3)}{|(g_i^{-1}e_3)\times (g_j^{-1}e_3)|},
\end{align*}
where $e_3$ is the standard basis $(0,0,1)^T$, $g_i^{-1}e_3$ denotes the third column of $g_i^{-1}$, and $|\cdot|$ denotes the Euclidean norm. 
Using the fact that $R(a\times b) =Ra \times Rb$ for $R\in \SO{3}$, we have 
\begin{align*}
    (\cos\alpha_{ij},\sin\alpha_{ij},0)^T &= \frac{e_3\times g_ig_j^{-1}e_3}{|e_3\times g_ig_j^{-1}e_3|}\\
    (\cos\alpha_{ji},\sin\alpha_{ji},0)^T &= \frac{(g_ig_j^{-1})^{-1}e_3\times e_3}{|(g_ig_j^{-1})^{-1}e_3\times e_3|},
\end{align*}
so that both $\alpha_{ij}$ and $\alpha_{ji}$ are only functions of the relative orientation $g_ig_j^{-1}$. 
The objective function $f_{ij}$ is then defined as 
\begin{align}\label{eq:fij}
    f_{ij}(g_ig_j^{-1}) = \|\widehat{I}_i(\cdot,\alpha_{ij}(g_ig_j^{-1})) - \widehat{I}_j (\cdot, \alpha_{ji}(g_ig_j^{-1}))\|_1,
\end{align}
where $\|\cdot\|_1$ denotes the $L^1$ norm\footnote{We remark that other loss functions besides the $L^1$ norm are also possible.}. 
We remark that when $i=j$, the term $\sum_{s\in S} f_{ii}(g_isg_i^{-1})$ in the loss function is not necessarily a constant as in the non-symmetric case, and captures the information contained in self common lines. 

The Fourier coefficients $\widehat{f}_{ij}(k)$ needed in forming the $C^{(k)}$ matrices can be computed by its definition 
\begin{align*}
    \widehat{f}_{ij}(k) = \int_{\SO{3}}  f_{ij}(g)\rho_k(g)^* dg,
\end{align*}
approximated by a quadrature formula on $\SO{3}$ \citep{kostelec2008ffts}. 
Precisely, let each $g\in \SO{3}$ be represented in terms of the ZYZ-Euler angles $(\alpha,\beta,\gamma)\in [0,2\pi]\times [0,\pi]\times[0,2\pi]$, discretized by the nodes 
\begin{align*}
    \alpha_{p} = \frac{2\pi p}{2T},\quad \beta_{\ell} =\frac{\pi(2\ell+1)}{4T},\quad \gamma_{q} = \frac{2\pi q}{2T},\quad 0\leq p,q,\ell\leq 2T-1
\end{align*}
for some integer $T$. 
Then $\widehat{f}_{ij}(k)$ can be approximated by 
\begin{align}\label{eq:quandrature}
    \widehat{f}_{ij}(k) \approx \frac{1}{(2T)^2} \sum_{p=0}^{2T-1}\sum_{q=0}^{2T-1}\sum_{\ell=0}^{2T-1} b_T(\ell)f_{ij}(\alpha_p,\beta_\ell,\gamma_q) \rho_k^*(\alpha_p,\beta_\ell,\gamma_q),
\end{align}
where 
\begin{align*}
    b_T(\ell) = \frac{2}{T}\sin \left[\frac{\pi(2\ell+1)}{4T}\right]\sum_{m=0}^{T-1}\frac{1}{2m+1}\sin\left[\frac{\pi(2\ell+1)(2m+1)}{4T}\right].
\end{align*}
It can be further shown that $f_{ij}$ as in \eqref{eq:fij} depends only on $\alpha$ and $\gamma$ \citep{bandeira2020non}.  
Using the fact $[\rho_k^*(\alpha,\beta,\gamma)]_{m',m} = e^{-\i m' \alpha} d_{m',m}^{(k)}(\beta)e^{-\i m\gamma}$, with $d_{m',m}^{(k)}$ being the Wigner-d matrix, we can simplify \eqref{eq:quandrature} as 
\begin{align*}
    [\widehat{f}_{ij}(k)]_{m',m} &\approx \frac{1}{(2T)^2} \sum_{p=0}^{2T-1}\sum_{q=0}^{2T-1}\sum_{\ell=0}^{2T-1} b_T(\ell)f_{ij}(\alpha_p,\gamma_q) e^{\i m' \alpha_p} d_{m',m}^{(k)}(\beta_\ell)e^{\i m\gamma_q}\\
    & = \frac{1}{(2T)^2} \left[\sum_{\ell=0}^{2T-1} b_T(\ell)d_{m',m}^{(k)}(\beta_\ell)\right]\left[\sum_{p=0}^{2T-1}\sum_{q=0}^{2T-1} e^{\i m' \alpha_p}f_{ij}(\alpha_p,\gamma_q) e^{\i m\gamma_q}\right],
\end{align*}
which can be computed in $O(T^3)$ steps.

\subsection{Truncation of Objective and Inequality Constraints via Fej\'er Kernel}\label{sec:truncation}
We notice that both the SDP objective in \eqref{eq:general sdp framework} and the inequality constraints \eqref{eq:inequality constraints} involve an infinite series that needs to be truncated for implementation. 
We shall simply truncate the objective function by $\sum_{k=1}^K \Tr[C^{(k)}X^{(k)}]$ for a suitable chosen $K$. 

The situation for the inequality constraint, however, requires more careful treatment since the truncated series does not necessarily need to be nonnegative, leaving the validity of such empirical constraint questionable. 
To resolve this issue, we follow \cite{bandeira2020non} by considering instead a Fej\'er kernel approximation to the Dirac delta function in \eqref{eq:delta function constraint} that shows up in the original form of the constraint.  
Precisely, the Fej\'er kernel over $\SO{3}$ is defined for a truncation level $K$ as 
\begin{align}\label{eq:fejer kernel}
    F_K (g) = F_K(\alpha,\beta,\gamma)& := \frac12 \sum_{k=0}^K(K-k+2)(K-k+1)(k+\frac12) \frac{\sin\left[(2k+1)\theta_g/2\right]}{\sin(\theta_g/2)},\quad g\in \SO{3},
\end{align}
where $\theta_g=2\arccos \big[\cos(\frac{\beta}{2})\cos(\frac{\alpha+\gamma}{2})\big]$. 
It arises as the second-order Ces\`aro sum of the Dirichlet kernel $D_k(\theta_g) = \sum_{m=-k}^k e^{\i m \theta_g}$ and is guaranteed to be nonnegative. 
The connection of \eqref{eq:fejer kernel} with the original expansion \eqref{eq:general framework Fourier expansion for delta} of the Dirac delta function is the following chain of identities 
\begin{align*}
    \Tr[\rho_k(g)]= D_k(\theta_g) = \frac{\sin\left[(2k+1)\theta_g/2\right]}{\sin(\theta_g/2)}
\end{align*}
so that we have 
\begin{align*}
    0\leq \frac12 \sum_{k=0}^K(K-k+2)(K-k+1)(k+\frac12) \Tr[\rho_k(g)] ,\quad \forall g\in \SO{3}. 
\end{align*}
Now replacing $g$ by $g^{-1}(g_isg_j^{-1})$ and averaging over $S$, we arrive at the constraint
\begin{align}
    0&\leq \frac{1}{|S|}\sum_{s\in S} F_K(g^{-1}(g_isg_j^{-1}))\nonumber\\
    &=\frac12 \sum_{k=0}^K(K-k+2)(K-k+1)(k+\frac12) \Tr[\rho_k(g)^* X_{ij}^{(k)}]\,\quad \forall g\in \SO{3}. \label{eq:fejer constraint 1}
\end{align}
The final step is then relaxing \eqref{eq:fejer constraint 1} to be satisfied only over a discrete approximation $\G_K$ of $\SO{3}$ based on the \emph{Hopf fibration} \citep{yershova2010generating}.

\subsection{ADMM}
For the actual computation of the SDP \eqref{eq:general sdp framework}, we shall adopt the alternating direction method of multipliers (ADMM) algorithm \citep{wen2010alternating}. 
In this subsection, we shall first review its general structure and then apply it to our specific setting. 
For a general SDP problem
\begin{align*}
    \underset{X}{\operatorname{minimize}}\quad  &\langle C,X\rangle\\
    \text{subject to }\,\,\,  &X\succeq 0 \\
      &\mathcal{A}_E(X) = b_E\\
     &\mathcal{A}_I(X) \geq b_I,
\end{align*}
where $\mathcal{A}_E(X)$ and $\mathcal{A}_I(X)$ involve inner products between certain symmetric matrices and $X$. 
By introducing several Lagrangian multipliers associated with constraints,
\begin{align*}
    -X \preceq 0  \quad   & \longleftrightarrow \quad  S\succeq 0 \\
    b_E-\AE(X) = 0 \quad   & \longleftrightarrow \quad y_E\\
    b_I-\AI(X) \leq 0 \quad   & \longleftrightarrow \quad y_I \geq 0,
\end{align*}
the Lagrangian for this problem has the form 
\begin{align*}
    \L(X;S,y_E,y_I)& = \langle C,X\rangle + \langle S,-X\rangle + \langle y_E, b_E-\AE(X) \rangle + \langle y_I,b_I-\AI(X)\rangle \\
    & = \langle y_E,b_E \rangle + \langle y_I,b_I\rangle +\langle X, C-S-\AE^*(y_E)-\AI^*(y_I)\rangle.  
\end{align*}
Moreover, the dual function takes the form 
\begin{align*}
    g(S,y_E,y_I)=\underset{X}{\operatorname{inf}} \,\, \L(X;S,y_E,y_I)
    = \begin{cases}
        \langle y_E,b_E \rangle + \langle y_I,b_I\rangle & \text{ if } C-S-\AE^*(y_E)-\AI^*(y_I)=0 \\
        -\infty & \text{ otherwise }
    \end{cases},
\end{align*}
leading to the dual problem  
\begin{align*}
    &\underset{S\succeq 0,y_E,y_I\geq 0}{\operatorname{maximize}}\,\, g(S,y_E,y_I)\\
    \Leftrightarrow \quad &\underset{S\succeq 0,y_E,y_I\geq 0}{\operatorname{maximize}}\,\,\langle y_E,b_E \rangle +\langle y_I,b_I\rangle &\\
    & \text{ s.t. }\quad S+\AE^*(y_E)+\AI^*(y_I)=C.
\end{align*}
The augmented Lagrangian takes the form
\begin{align*}
    \mathcal{L}_\rho(S,y_E,y_I;X) = - \langle y_E, b_E\rangle - \langle y_I, b_I\rangle  +\langle  X,S+\AE^*(y_E)+\AI^*(y_I)-C\rangle   \\+\frac{\rho}{2}\|S+\AE^*(y_E)+\AI^*(y_I)-C\|_F^2 
\end{align*}
and we shall minimize it (over $S\succeq 0,y_E,y_I\geq 0$ sequentially). 
The ADMM scheme then takes the following steps:
\begin{itemize}
    \item Update $S$. We perform a projection onto the PSD cone, that is
    \begin{align*}
        S^{(k+1)} &= \left[C-\AE^*(y_E^{(k)})-\AI^*(y_I^{(k)})-{X^{(k)}}/{\rho} \right]_{\succeq0}.
    \end{align*}

    \item Update $y_E$. We set the gradient $\nabla_{y_E} \L_{\rho}$ to be $0$ and solve for $y_E^{(k+1)}$, which gives
    \begin{align*}
        \nabla_{y_E} \L_{\rho} &=\AE(X)-b_E +\rho \Big[\AE\AE^*(y_E)+\AE\Big(S-C+\AI^*(y_I)\Big)\Big] = 0\\
        \Rightarrow \quad y_E^{(k+1)}&=(\AE\AE^*)^{-1}\left[\frac{b_E-\AE(X^{(k)})}{\rho} +\AE \Big(C-S^{(k+1)}-\AI^*(y_I^{(k)})\Big)\right]\\
        & =(\AE\AE^*)^{-1}\left[\frac{b_E}{\rho} +\AE \left(-\frac{X^{(k)}}{\rho}+C-S^{(k+1)}-\AI^*(y_I^{(k)})\right)\right].
    \end{align*}
    In the non-symmetric NUG case, $\AE\AE^*$ is simply the identity. 

    \item Update $y_I$. If we directly solve $0=\nabla_{y_I}\L_\rho$, then we would need to compute $(\AI\AI^*)^{-1}$, which is computationally quite expensive. 
    Instead, we consider $\L_\rho + \frac{\rho}{2}\|y_I-y_I^{(k)}\|^2_{\lambda \I-\AI\AI^*}$, where $\lambda$ is the (approximately) largest eigenvalue of $\AI\AI^*$.
    \begin{align*}
        &\quad \L_\rho + \frac{\rho}{2}\|y_I-y_I^{(k)}\|^2_{\lambda \I-\AI\AI^*}\\
        &=-\langle b_I,y_I \rangle +
        \langle \AI(X),y_I\rangle+
        \rho \left\langle y_I, \AI\Big(S-C+\AE^*(y_E)+\AI^*(y_I^{(k)})\Big) \right\rangle \\
        &+ \frac{\lambda\rho}{2}\langle y_I,y_I \rangle -\lambda\rho \langle y_I,y_I^{(k)}\rangle  +\text{ a term independent of } y_I \\
        & \quad\nabla_{y_I}\left[\L_\rho +\frac{\rho}{2}\|y_I-y_I^{(k)}\|^2_{\lambda \I-\AI\AI^*}\right]\\
        &=\lambda \rho y_I -\lambda \rho y_I^{(k)}-b_I+\AI(X) + \rho \left[ \AI\Big(S-C+\AE^*(y_E)+\AI^*(y_I^{(k)})\Big) \right]\\
        y_I^{(k+1)}&=\left[\frac{b_I-\AI(X^{(k)})}{\lambda\rho} +y_I^{(k)} +\frac{1}{\lambda}\AI\Big(C-S^{(k+1)}-\AE^*(y_E^{(k+1)})-\AI^*(y_I^{(k)})\Big)\right]_{\geq 0}.
    \end{align*}
    We can perform multiple updates for $y_I$ to improve the performance.

    \item Update the multiplier $X$. 
    \begin{align*}
        X^{(k+1)}=X^{(k)} +\rho\Big[S^{(k+1)}+\AE^*(y_E^{(k+1)})+\AI^*(y_I^{(k+1)})-C\Big]. 
    \end{align*}
\end{itemize}
We use the randomly permutated order to update the primal variables to ensure convergence in practice; see \cite{sun2020efficiency}. 

\paragraph{Application to our setting:}

Consider 
\begin{align*}
    \mathbb{X}=\begin{bmatrix}
    X & & \\
     & X_d &\\
     & & X_q
    \end{bmatrix},
\end{align*}
where $X$ is the block diagonal matrix containing all $X^{(k)}$, $X_d$ lists all diagonal blocks of $I_{d_k}-\XX{i}{i}{k}$, and $X_q$ contains all transformations $Q(\XX{i}{j}{1})$ (cf. Section \ref{subsec: so3convexhull}). 
This then leads to the following sets of more explicit constraints.  
\begin{equation}\tag{Equality constraints}
\begin{aligned}
    \Tr[\XX{i}{i}{k}]&=\text{rank}(A_k)\qquad \forall k,i,\\
    \XX{i}{i}{k}+ [X_d]_{i}^{(k)} &= I_{d_k} \qquad \forall k,i,\\
    Q(\XX{i}{j}{1})-[X_q]_{ij}&=0 \qquad \forall i,j.
\end{aligned}
\end{equation}
\begin{equation}\tag{PSD constraints}
\begin{aligned}
    X^{(k)} \succeq 0,\quad [X_d]_i^{(k)} \succeq 0,\quad [X_q]_{ij} \succeq 0 \qquad \forall k,i,j.
\end{aligned}
\end{equation}
\begin{equation}\tag{Inequality constraints}
\begin{aligned}
    \frac12(K-2)(K-1)+\sum_{k=1}^K(K-k+2)(K-k+1)(k+\frac12) \Tr[W^{(k)}(\alpha,\beta,\gamma)^*X_{ij}^{(k)}] \geq 0, \quad \forall (\alpha,\beta,\gamma)\in\G_K\,\quad \forall i,j.
\end{aligned}
\end{equation}
Collecting these constraints altogether leads to the actual SDP that we shall solve for computation. 

\section{Numerical Experiments}\label{sec:numerics}
In this section, we shall demonstrate the effectiveness of our symmetry NUG framework to the cryo-EM problem on both simulated and real datasets. 
The code and data will be made publicly available upon publication through the ASPIRE open-source package \citep{garrett_wright_2023_8321443}. 

\subsection{Simulation Studies}
To start with, we consider simulated datasets generated from the base molecule \href{https://www.ebi.ac.uk/emdb/EMD-2660}{EMD-2660} taken from the EMDB database \citep{lawson2016emdatabank}.
The original EMD-2660 molecule $\phi$ does not have any symmetry, so we shall generate its symmetric versions by averaging itself with its symmetric transformations. 
Precisely, for a symmetry group $S$, generate
\begin{align*}
    \phi_S = \frac{1}{|S|}\sum_{s\in S} s\circ \phi,
\end{align*}
which then has $S$ as its symmetry group. 

Now given a molecule $\phi_S: \R^3\rightarrow \R$, we shall generate projection images by first sampling random rotation matrices $\{R_i\}_{i=1}^N\subset \SO{3}$, and then define
\begin{align*}
    I_i = P(R_i\circ\phi_S) + \epsilon_i,
\end{align*}
where we recall that $P$ is the tomographic projection operator and  $\epsilon_i$ is entry-wise Gaussian noise with variance $\sigma^2$. 
In particular, we are assuming perfectly centered images and ignoring the CTF effects here.  
In this section, we shall use the ASPIRE-Python package \citep{garrett_wright_2023_8321443} for generating the projection images $I_i\in \R^{L\times L}$, whose size is fixed to be $L=128$. 

We shall test the performance of the NUG algorithm on recovering $\{R_i\}_{i=1}^N$ for different symmetry groups and signal-to-noise ratios (SNR), defined as 
\begin{align*}
    \text{SNR } = \frac{\text{Var(Signal)}}{\text{Var(Noise)}} .
\end{align*}
Figure \ref{fig:image SNR} shows visualization of images with different SNRs generated by the original non-symmetric molecule EMD-2660. 
The truncation level as discussed in Section \ref{sec:truncation} is set to be $K=12$ throughout, with $\G_K$ containing 16317 elements that approximately form a uniform grid over $\SO{3}$.
We remark that a larger truncation level leads to better approximation of the objective function, but at the cost of a larger SDP that is more costly and prone to numerical error. 
The choice $K=12$ is tuned empirically to achieve such a balance, beyond which no improvement is attained, and is observed to work well for all the considered cyclic and dihedral examples. 

We perform proximal refinement as in Section \ref{sec:proximal} with 4 iterations, by projecting the solutions to be of rank $2\cdot\operatorname{rank}(A_k)+3$, $2\cdot\operatorname{rank}(A_k)+2$, $2\cdot\operatorname{rank}(A_k)+1$, $2\cdot\operatorname{rank}(A_k)$ respectively. The penalty $\eta_k$ is set to be $\frac{1}{k}$. 
\begin{figure}[!htb]
    \centering
    \minipage{0.2\textwidth}
    \includegraphics[width=\linewidth]{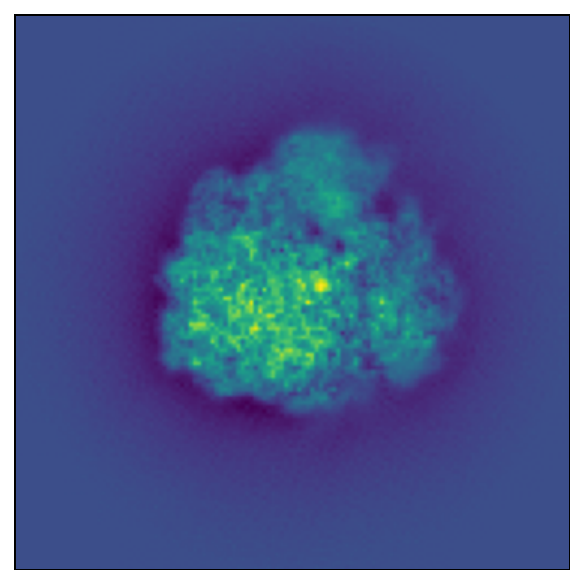}
    \subcaption*{Clean}
    \endminipage\hfill
    \minipage{0.2\textwidth}
    \includegraphics[width=\linewidth]{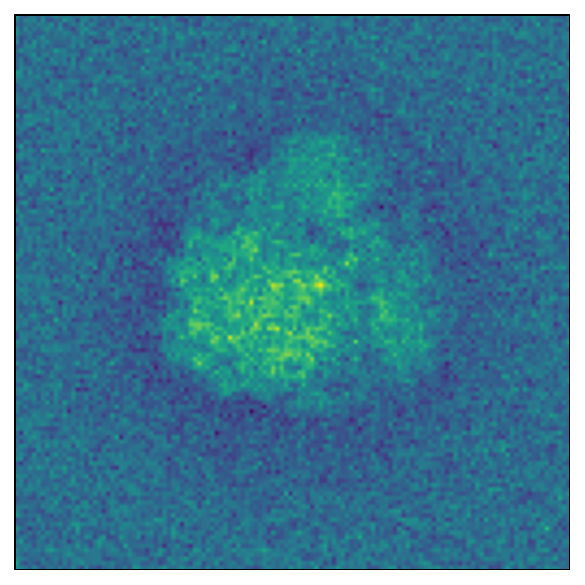}
    \subcaption*{SNR=1}
    \endminipage\hfill
    \minipage{0.2\textwidth}
    \includegraphics[width=\linewidth]{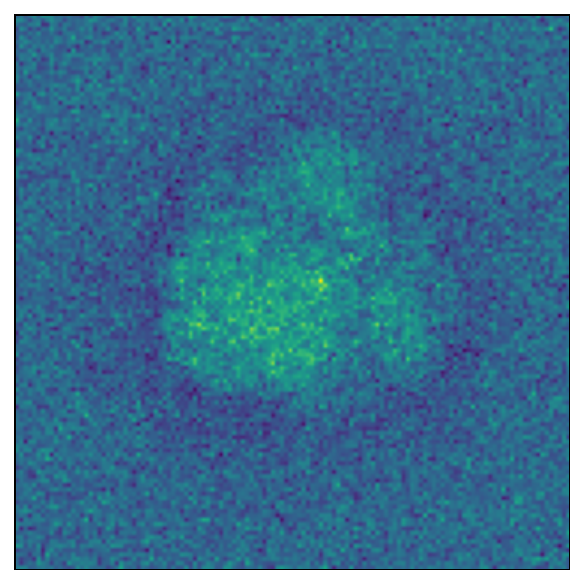}
    \subcaption*{SNR=1/2}
    \endminipage\hfill
    \minipage{0.2\textwidth}
    \includegraphics[width=\linewidth]{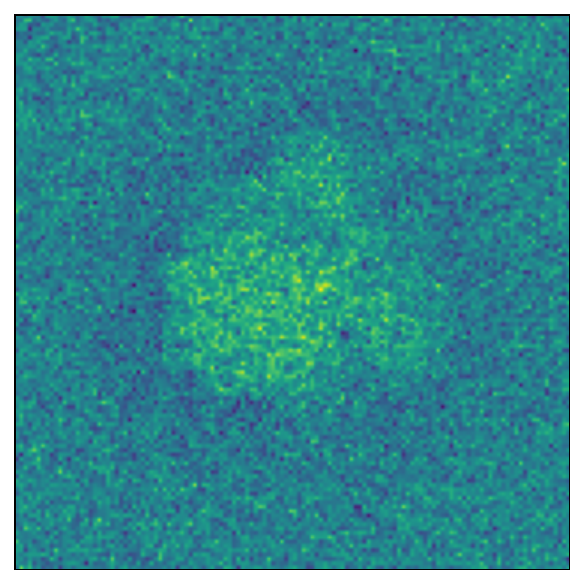}
    \subcaption*{SNR=1/4}
    \endminipage\hfill
    \minipage{0.2\textwidth}
    \includegraphics[width=\linewidth]{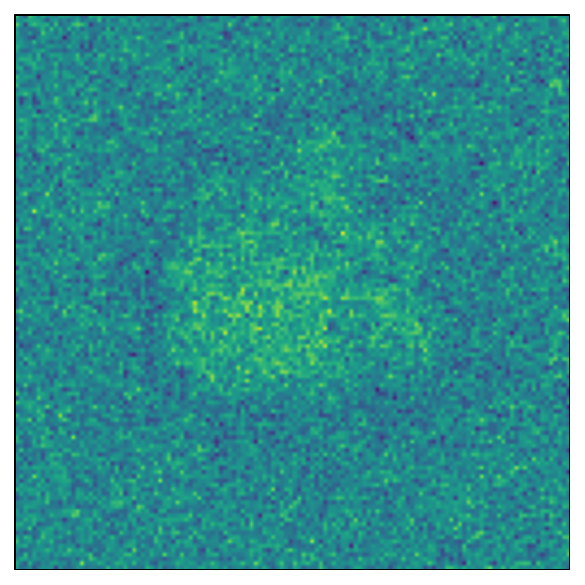}
    \subcaption*{SNR=1/8}
    \endminipage\hfill
    \caption{Simulated projection images of EMD-2660 at different SNR values.}\label{fig:image SNR}
\end{figure}
Denoting the estimated rotations as $\{\widehat{R}_i\}_{i=1}^N$, we shall quantify the performance via the mean squared error (MSE) of estimating all pairwise rotations
\begin{equation}\label{eq:MSE cryoem}
\begin{aligned}
    \text{MSE } &= \operatorname{min} \left\{\frac{1}{N^2}  \sum_{i,j=1}^N \underset{s\in S}{\operatorname{min}}\,\, \|R_i^{-1}R_j-\widehat{R}_i^{-1}s \widehat{R}_j\|_F^2 ,\frac{1}{N^2}  \sum_{i,j=1}^N \underset{s\in JSJ^*}{\operatorname{min}}\,\, \|R_i^{-1}R_j-(J\widehat{R}_iJ^*)^{-1}s (J\widehat{R}_jJ^*)\|_F^2 \right\}\\
    &=\operatorname{min} \left\{\frac{1}{N^2}  \sum_{i,j=1}^N \underset{s\in S}{\operatorname{min}}\,\, \|R_i^{-1}R_j-\widehat{R}_i^{-1}s \widehat{R}_j\|_F^2 ,\frac{1}{N^2}  \sum_{i,j=1}^N \underset{s\in S}{\operatorname{min}}\,\, \|R_i^{-1}R_j-J\widehat{R}_i^{-1}s \widehat{R}_jJ^*\|_F^2 \right\}
\end{aligned}
\end{equation}
where we recall that $J$ is the reflection matrix \eqref{eq:reflection matrix}. 

\begin{remark}
Since we can only recover the rotations up to symmetry and a global rotation, $\widehat{R}_i$ is only an estimate of $RJ^{\delta}s_iR_iJ^{\delta}$, where $s_i\in S$ can be different for each $i$, $R\in \SO{3}$ is the global rotation, and $\delta\in \{0,1\}$ represents the $J$-ambiguity.
Ideally, we should quantify the error as 
\begin{align*}
    \operatorname{min}\left\{\underset{R \in \SO{3}}{\operatorname{min}}\frac{1}{N} \sum_{i=1}^N\underset{s\in S}{\operatorname{min}} \|R_i-sR\widehat{R}_i\|_F^2, \underset{R \in \SO{3}}{\operatorname{min}}\frac{1}{N} \sum_{i=1}^N\underset{s\in S}{\operatorname{min}} \|R_i-sJR\widehat{R}_iJ\|_F^2\right\} 
\end{align*}
However, it is not obvious that there is an efficient way of calculating such error, so we resort to the MSE of estimating relative rotations defined in \eqref{eq:MSE cryoem} that sidesteps the determination of the global rotation $R$. \qed 
\end{remark}

\begin{table}[!htb]
    \centering
    \begin{tabular}{|c|c|c|c|c|c|}\hline
         SNR&  Clean&  1&  1/2&  1/4& 1/8\\\hline\hline 
 $C_1$ (without PR)& 0.030 (18.10)& 0.041 (30.53)& \textbf{0.034} (33.50)& \textbf{0.088} (44.10)&\textbf{0.570} (62.08)\\\hline
 Synchronization & \textbf{0.000} (7.58)& \textbf{0.005} (23.25)& 0.046 (33.27)& $>1$&$>1$\\\hline \hline
         $C_2$ (without PR)&  0.032 (16.60)&  0.026 (25.61)&  0.047 (32.12)&  0.195 (49.83)&  $>1$\\\hline
         $C_2$ (with PR)&  \textbf{0.017} (12.73)&  \textbf{0.017} (23.67)&  \textbf{0.027} (29.58)&  \textbf{0.043} (37.89)& $>1$\\\hline 
         Synchronization & 0.521 (8.29)& 0.684 (40.00)& $>1$& $>1$& $>1$\\\hline\hline
         $C_3$ (without PR)&   0.026 (14.53)&  0.033 (24.38)&  0.096 (38.16)&  0.852 (60.91)& $>1$\\\hline
         $C_3$ (with PR)&  0.017 (11.58)&  \textbf{0.018} (22.17)&  \textbf{0.037} (30.71)&  \textbf{0.054} (38.56)&  $>1$\\\hline 
         Synchronization & \textbf{0.004} (7.54)& $>1$& $>1$& $>1$&  $>1$\\\hline\hline
         $C_4$ (without PR)&  0.153 (35.84)&  0.370 (60.91)&  0.804 (64.75)&  $>1$& $>1$\\\hline
         $C_4$ (with PR)&  0.050 (20.47)&  \textbf{0.046} (27.19)&  \textbf{0.457} (35.47)&  $>1$& $>1$\\\hline 
         Synchronization & \textbf{0.005} (7.54)& $>1$& $>1$& $>1$& $>1$\\\hline\hline
         $C_5$ (without PR)&  0.190 (39.41)&   0.189 (41.66)&  0.724 (65.54)&  $>1$& $>1$\\\hline
         $C_5$ (with PR)&  0.060 (20.20)&  0.065 (27.88)&  \textbf{0.338} (48.10)&  $>1$& $>1$\\\hline 
         Synchronization & \textbf{0.006} (7.54)& \textbf{0.063} (27.41)& $>1$& $>1$& $>1$\\\hline\hline
         $C_6$ (without PR)& 0.442 (50.99)& 0.908 (65.63)& $>1$& $>1$&$>1$\\\hline
         $C_6$ (with PR)& 0.105 (29.47)& \textbf{0.100} (65.19)& \textbf{0.622} (66.08)& $>1$&$>1$\\\hline 
         Synchronization & \textbf{0.006} (7.54)& $>1$& $>1$& $>1$& $>1$\\ \hline
    \end{tabular}
    \caption{MSE \eqref{eq:MSE cryoem} with $N=40$ images of size $L=128$ for different symmetry groups and SNRs. FSCs of the reconstructed volumes are reported in the parenthesis in angstrom. 
    Here PR refers to the proximal refinement step as in Section \ref{sec:proximal}.}
    \label{table:mse cryoem n=40}
\end{table}

\begin{table}[!htb]
    \centering
    \begin{tabular}{|c|c|c|c|c|c|}\hline
         SNR&  Clean&  1&  1/2&  1/4& 1/8\\\hline \hline 
 $C_1$ (without PR)& 0.022 (14.28)& 0.028 (21.76)& 0.039 (26.59)& \textbf{0.076} (34.91)&\textbf{0.398} (46.79)\\\hline
 Synchronization & \textbf{0.000} (7.54)& \textbf{0.003} (18.53)& \textbf{0.021} (24.55)& 0.150 (39.25)&$>1$\\\hline\hline 
         $C_2$ (without PR)&  0.019 (11.06)&  0.021 (20.58)&  0.043 (25.22)&  0.102 (35.73)&  $>1$\\\hline
         $C_2$ (with PR)&  0.013 (8.92)&  \textbf{0.017} (20.25)&  \textbf{0.027} (22.75)&  \textbf{0.035} (27.04)& $>1$\\\hline 
         Synchronization & \textbf{0.003} (7.54)& 0.248 (24.35)& 0.661 (42.46)& $>1$& $>1$\\\hline\hline 
         $C_3$ (without PR)&   0.036 (14.64)&  0.032 (20.82)&  0.138 (33.92)&  0.568 (50.78)& $>1$\\\hline
         $C_3$ (with PR)&  0.025 (10.13)&   \textbf{0.019} (19.40)&  \textbf{0.023} (22.19)&  \textbf{0.039} (27.68)& $>1$\\\hline 
         Synchronization & \textbf{0.004} (7.54)& $>1$& $>1$& $>1$& $>1$\\\hline\hline 
         $C_4$ (without PR)&   0.082 (16.27)&  0.242 (38.26)&  0.598 (60.07)&  $>1$& $>1$\\\hline
         $C_4$ (with PR)&  0.037 (11.92)&   \textbf{0.051} (21.44)&  \textbf{0.045} (24.14)&  $>1$& $>1$\\\hline 
         Synchronization & \textbf{0.001} (7.54)& $>1$& $>1$& $>1$& $>1$\\\hline\hline 
         $C_5$ (without PR)&  0.094 (17.85)&  0.273 (41.37)&  $>1$&  $>1$& $>1$\\\hline
         $C_5$ (with PR)&  0.050 (10.46)&  0.069 (21.50)&  $>1$&  $>1$& $>1$\\\hline 
         Synchronization & \textbf{0.007} (7.54)& \textbf{0.049} (20.28)& $>1$& $>1$& $>1$\\\hline\hline 
         $C_6$ (without PR)& 0.238 (29.61)& 0.737 (61.37)& $>1$& $>1$&$>1$\\\hline
         $C_6$ (with PR)& 0.093 (14.82)& \textbf{0.174} (47.39)& \textbf{0.487} (66.63)& $>1$&$>1$\\\hline 
         Synchronization & \textbf{0.010} (7.54)& $>1$& $>1$& $>1$& $>1$\\ \hline
    \end{tabular}
    \caption{MSE \eqref{eq:MSE cryoem} with $N=100$ images of size $L=128$ for different symmetry groups and SNRs. FSCs of the reconstructed volumes are reported in the parenthesis in angstrom. 
    Here PR refers to the proximal refinement step as in Section \ref{sec:proximal}.}
    \label{table:mse cryoem n=100}
\end{table}

Tables \ref{table:mse cryoem n=40} and \ref{table:mse cryoem n=100} compare the performance of our proposed approach with the synchronization based approach proposed in \cite{pragier2019common}.
The comparisons run across cyclic symmetries from $C_1$ up to $C_6$ with $N=40$ and $N=100$, in terms of both the MSE \eqref{eq:MSE cryoem} and the resolution as determined by the Fourier Shell Correlation (FSC) of the resulting reconstructed volumes. 
We see that our algorithm returns fairly accurate viewing direction estimations for moderate noise levels and symmetry orders, and is more robust to noise.
The synchronization method outperforms ours in the clean case as expected in that synchronization exploits exact common lines geometry that becomes more accurate in the noiseless setting than the truncation-relaxation-rounding approximation in SDP approaches. 
If instead the original NUG algorithm \citep{bandeira2020non} is applied to these symmetric molecules, we have observed that large MSEs are incurred even for clean images and the reconstructed molecule look visually very different from the ground truth.

We also note that the proximal refinement step is indeed able to improve the accuracy of the initial estimates, even by a significant amount in certain cases.
The performance starts to degrade as noise level becomes high and symmetry order is large. 
We believe this is due to the fact that our SDP relaxation becomes less tight in such cases since we are dropping many low rank constraints that should be imposed on the diagonal blocks, which is exacerbated by high noise levels. 
Nevertheless, we shall demonstrate on real datasets in the next subsection that when working with class averages which have much smaller noise level, our NUG framework is able to produce good ab initio models for symmetric molecules.  
Tables \ref{table:mse cryoem D n=40} and \ref{table:mse cryoem D n=100} record the performance of our NUG approach for dihedral symmetries, where similar observations are made.  
We note that \cite{rosen2020common} considers orientation estimation for molecules with $D_2$ symmetry and, to the best of our knowledge, is the only prior work addressing dihedral symmetry in this context. Our algorithm therefore provides a useful complementary approach.

\begin{table}[!htb]
    \centering
    \begin{tabular}{|c|c|c|c|c|c|}\hline
         SNR&  Clean&  1&  1/2&  1/4& 1/8\\\hline\hline
         $D_3$ (without PR)& 0.049 (35.52)&  0.035 (52.32)& 0.204 (61.37)& 0.684 (64.93)&$>1$\\\hline
         $D_3$ (with PR)&  0.040 (17.86)& 0.027 (51.54)& 0.123 (60.68)& 0.448 (60.07)&$>1$\\ \hline\hline
         $D_4$ (without PR)& 0.158 (60.60)&  0.220 (61.69)& 0.576 (43.23)& $>1$&$>1$\\\hline
         $D_4$ (with PR)& 0.082 (19.43)& 0.097 (50.67)& 0.511 (38.68)& $>1$&$>1$\\ \hline
         \hline
         $D_5$ (without PR)& 0.101 (12.52)&   0.180 (61.69)& $>1$& $>1$&$>1$\\\hline
         $D_5$ (with PR)& 0.088 (12.15)& 0.167 (49.68)& $>1$&  $>1$&$>1$\\ \hline\hline
         $D_6$ (without PR)& 0.089 (16.57)& $>1$& $>1$& $>1$& $>1$\\\hline
         $D_6$ (with PR)& 0.391 (47.06)& $>1$& $>1$&$>1$& $>1$\\ \hline
\end{tabular}
    \caption{MSE \eqref{eq:MSE cryoem} with $N=40$ images of size $L=128$ for different symmetry groups and SNRs. FSCs of the reconstructed volumes are reported in the parenthesis in angstrom. 
    Here PR refers to the proximal refinement step as in Section \ref{sec:proximal}.}
    \label{table:mse cryoem D n=40}
\end{table}

\begin{table}[!htb]
    \centering
    \begin{tabular}{|c|c|c|c|c|c|}\hline
         SNR&  Clean&  1&  1/2&  1/4& 1/8\\\hline\hline
         $D_3$ (without PR)& 0.085 (12.72)&   0.064 (44.71)& 0.118 (52.61)& 0.903 (50.78)&$>1$\\\hline
         $D_3$ (with PR)& 0.082 (16.51)& 0.072 (28.61)& 0.096 (47.53)& 0.646 (58.90)&$>1$\\ \hline\hline
         $D_4$ (without PR)& 0.029 (14.66)& 0.160 (27.36)& 0.455 (62.33)& $>1$&$>1$\\\hline
         $D_4$ (with PR)& 0.024 (16.71)& 0.184 (26.25)& 0.456 (62.08)&  $>1$&$>1$\\ \hline
         \hline
         $D_5$ (without PR)& 0.110 (12.65)&   0.185 (45.86)& 0.808 (53.78)& $>1$&$>1$\\\hline
         $D_5$ (with PR)& 0.109 (16.95)& 0.125 (44.83)& 0.771 (42.61)&  $>1$&$>1$\\ \hline\hline
         $D_6$ (without PR)& 0.319 (19.80)& $>1$& $>1$& $>1$& $>1$\\\hline
         $D_6$ (with PR)& 0.566 (26.99)& $>1$& $>1$&$>1$&$>1$\\ \hline
\end{tabular}
    \caption{MSE \eqref{eq:MSE cryoem} with $N=100$ images of size $L=128$ for different symmetry groups and SNRs. FSCs of the reconstructed volumes are reported in the parenthesis in angstrom. 
    Here PR refers to the proximal refinement step as in Section \ref{sec:proximal}.}
    \label{table:mse cryoem D n=100}
\end{table}

To close this section, we remark that one could also exclude the effects of self common lines by removing the diagonal losses $f_{ii}$'s (cf. \eqref{eq:fij}) from the objective function. 
Our experiments suggest that whether the self common lines are excluded or not have small effects on the algorithmic performance. 
Lastly, we present in Table \ref{tab:time} the computational time for solving the SDP \eqref{eq:general sdp framework} for different number of images, where the experiments are run on one A100 GPU with 80G memory. 
The reported runtime corresponds to the $C_2$ molecule, with comparable runtimes observed for the other symmetries tested.  
In particular, the reported runtimes exclude proximal refinement. The total runtime should therefore be multiplied by five, since the refinement step requires four additional SDP solves. 
We can see that the computational time scales roughly quadratically with the number of images. 
The running times for the synchronization based algorithm depend on the noise level as well as the symmetry order due to specialized algorithms for $C_2,C_3,C_4$ symmetries. 
On a 2.8 GHz Intel CPU, they can vary between 10$\sim$100 minutes for $N=100$ images and 110$\sim$1100 minutes for $N=300$ images.

\begin{table}[!htb]
    \centering
    \begin{tabular}{|c|c|c|c|c|c|}\hline
         $N$ & 20 & 50&  100&  200& 300\\\hline
         Time (minutes)& 2.0 &  6.8&  21.3&  78.7& 176.2\\ \hline
    \end{tabular}
    \vspace{-5pt}
    \caption{Computation time of solving the SDP \eqref{eq:general sdp framework} for different number of images with truncation level $K=12$. }
    \label{tab:time}
\end{table}

\subsection{Experimental Datasets}

\subsubsection{EMPIAR-10081}
In this section, we shall apply our NUG framework on a real dataset with $C_4$ symmetry taken from the \href{https://www.ebi.ac.uk/empiar/EMPIAR-10081/}{EMPIAR-10081} entry from the EMPIAR image database \citep{iudin2016empiar} that contains 55870 raw images. 
We use RELION \citep{scheres2012relion} for generating 50 class averages from the original dataset. 
Figure \ref{fig:10081 images} shows several raw images from EMPIAR and the computed class averages. 
\begin{figure}[!htb]
    \centering
    \minipage{0.2\textwidth}
    \includegraphics[width=\linewidth]{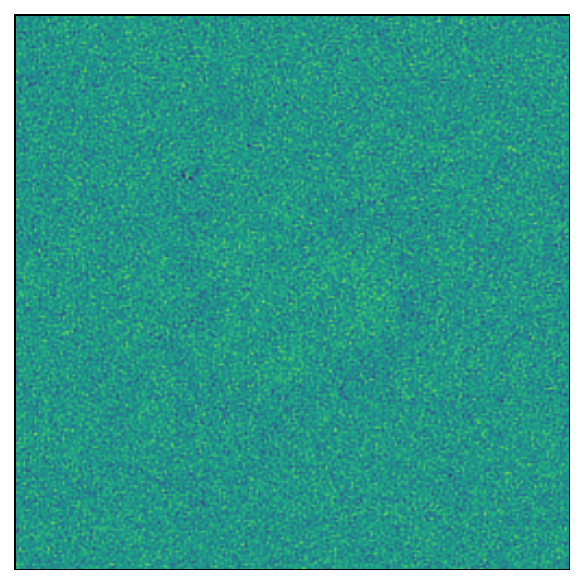}
    \endminipage\hfill
    \minipage{0.2\textwidth}
    \includegraphics[width=\linewidth]{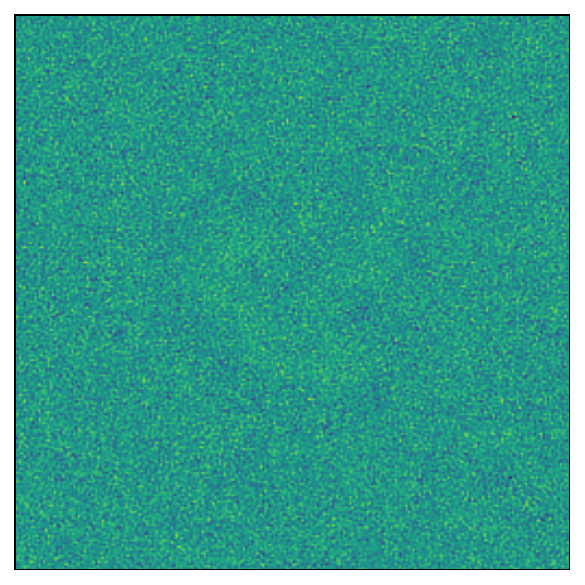}
    \endminipage\hfill
    \minipage{0.2\textwidth}
    \includegraphics[width=\linewidth]{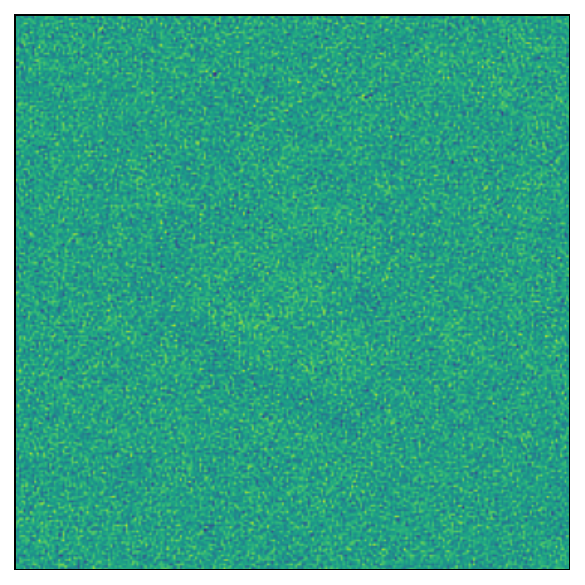}
    \endminipage\hfill
    \minipage{0.2\textwidth}
    \includegraphics[width=\linewidth]{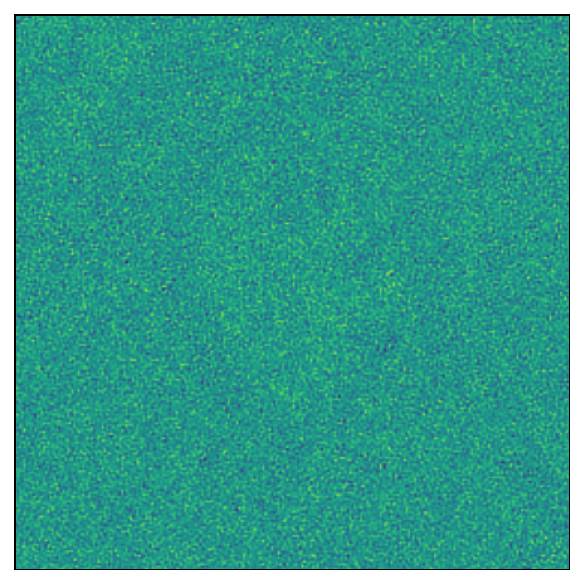}
    \endminipage\hfill
    \minipage{0.2\textwidth}
    \includegraphics[width=\linewidth]{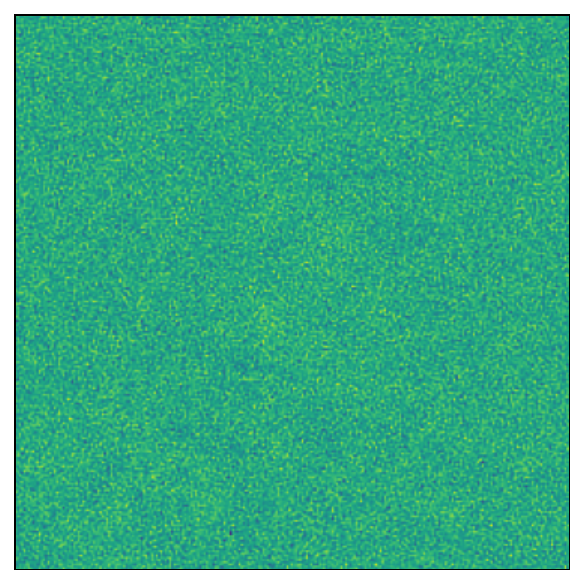}
    \endminipage\hfill

    \minipage{0.2\textwidth}
    \includegraphics[width=\linewidth]{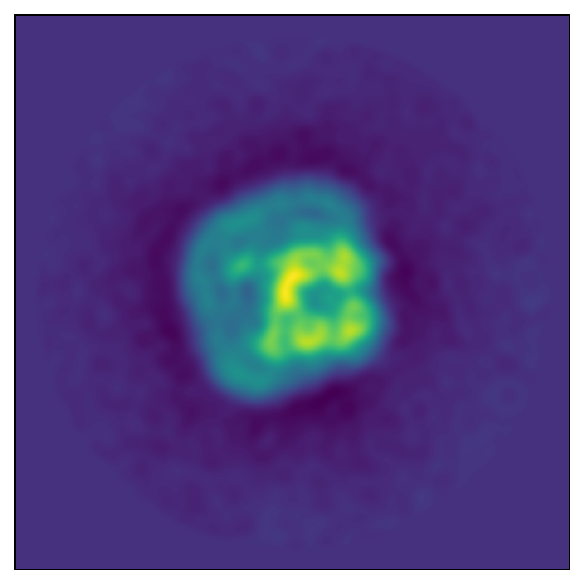}
    \endminipage\hfill
    \minipage{0.2\textwidth}
    \includegraphics[width=\linewidth]{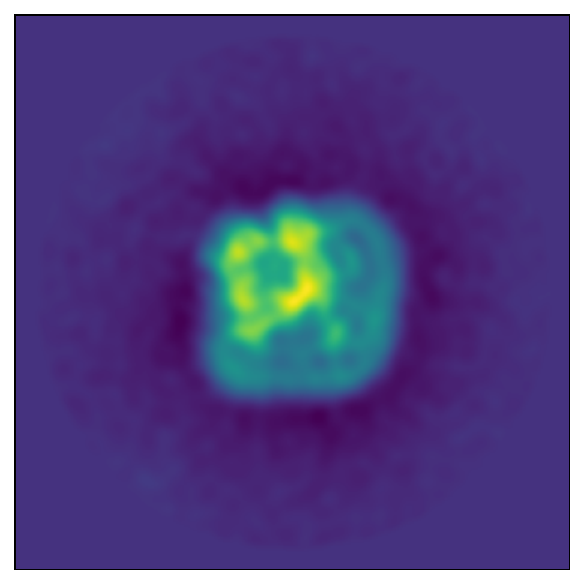}
    \endminipage\hfill
    \minipage{0.2\textwidth}
    \includegraphics[width=\linewidth]{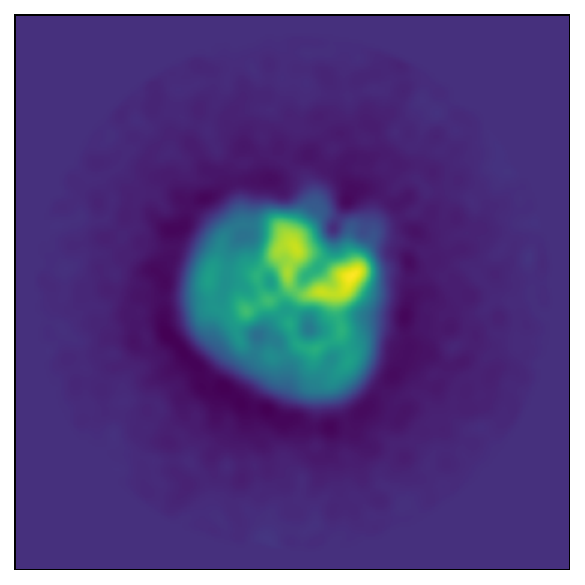}
    \endminipage\hfill
    \minipage{0.2\textwidth}
    \includegraphics[width=\linewidth]{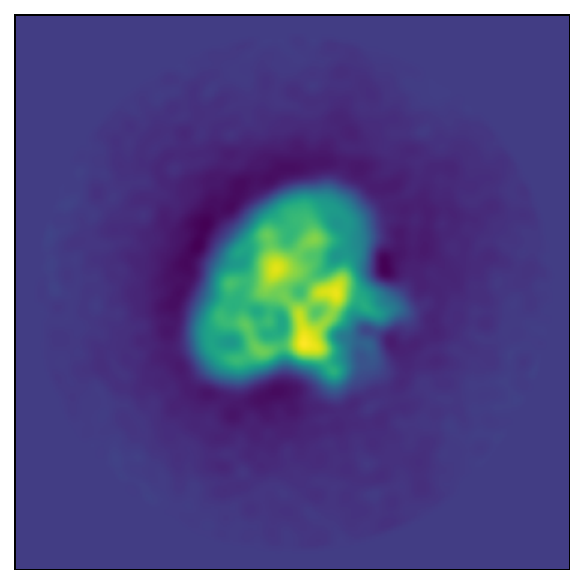}
    \endminipage\hfill
    \minipage{0.2\textwidth}
    \includegraphics[width=\linewidth]{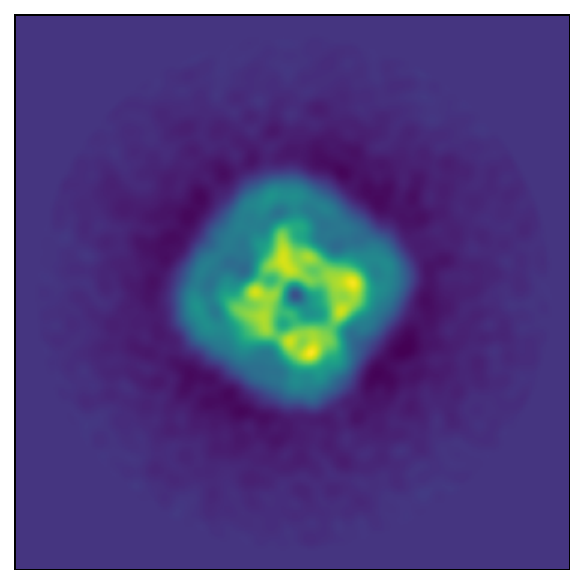}
    \endminipage\hfill
    \caption{Raw images from EMPIAR 10081 (top) and class averages computed (bottom).}\label{fig:10081 images}
\end{figure}
Out of the 50 class averages, we manually pick 17 of them that look visually meaningful, which are of size 256 by 256.

We apply our NUG algorithm to obtain the orientation estimations for these 17 class averages, with truncation level $K=12$ and 10 iterations of proximal refinement where in each iteration the projection rank is set to be $\operatorname{rank}(A_k)$ and the penalty is set to be $\eta_k=\frac{1}{k}$. 
The total computational time is about 20 minutes on an A100 GPU with 80G memory. 
We then use the estimated orientations to build an ab initio model as displayed in Figure \ref{fig:10081-abinitio 3d}.
We can see that even with only 17 images, our algorithm is able to yield a reasonably good reconstruction with a Fourier shell correlation (FSC) 15.45$\mathring{A}$ with respect to the reference model (which we take as the reconstructed volume reported under the EMPIAR-10081 entry) as is shown in Figure \ref{fig:10081-abinitio fsc}.

\begin{figure}[!htb]
    \centering
    \minipage{0.33\textwidth}
    \includegraphics[width=\linewidth]{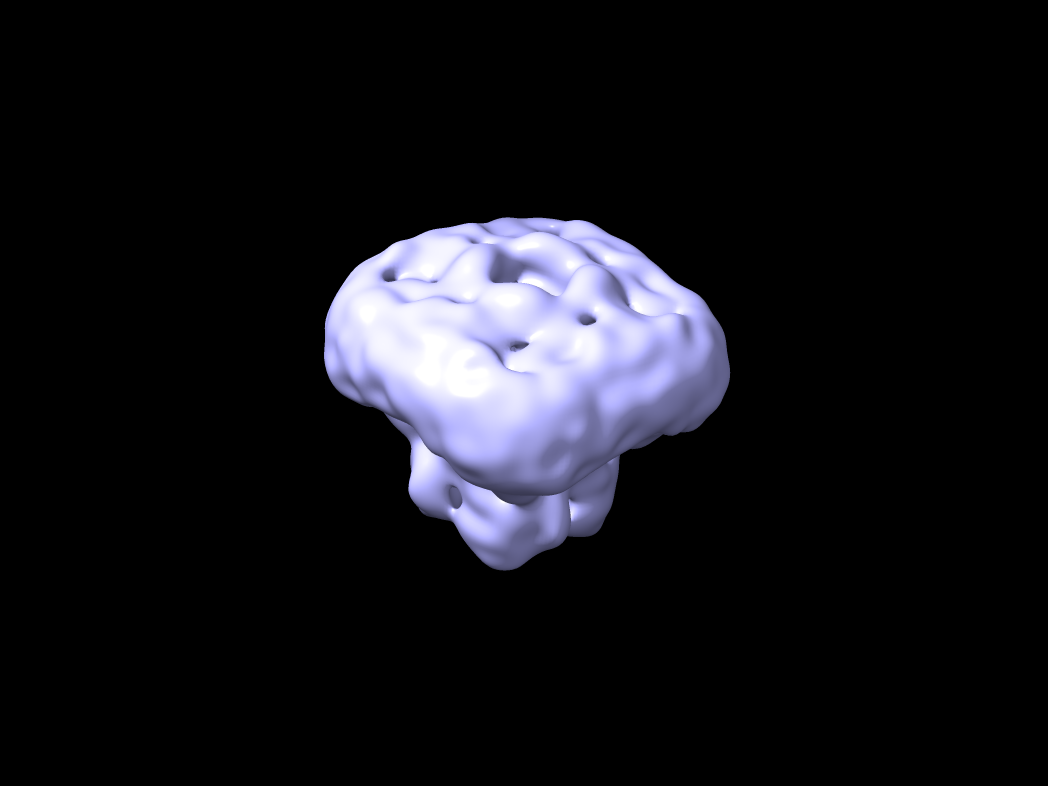}
    \subcaption{Ab initio model}\label{fig:10081-abinitio 3d}
    \endminipage\hfill
    \minipage{0.33\textwidth}
    \includegraphics[width=\linewidth]{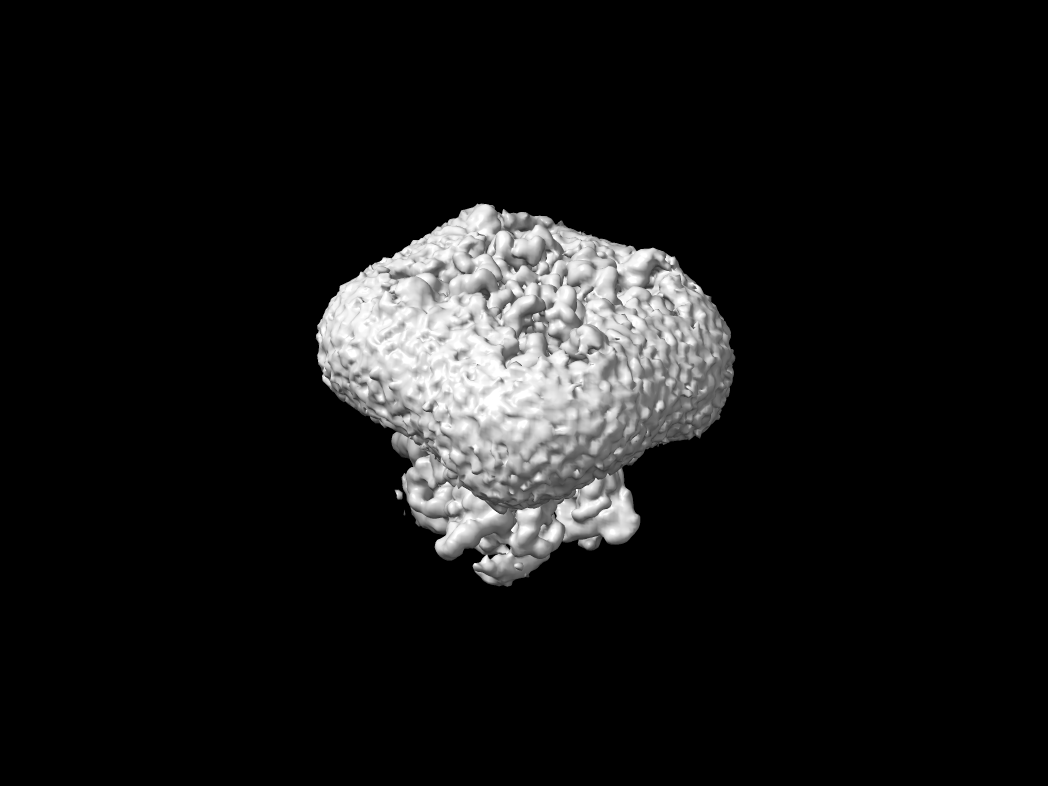}
    \subcaption{Refined model}\label{fig:10081-refine 3d}
    \endminipage\hfill
    \minipage{0.33\textwidth}
    \includegraphics[width=\linewidth]{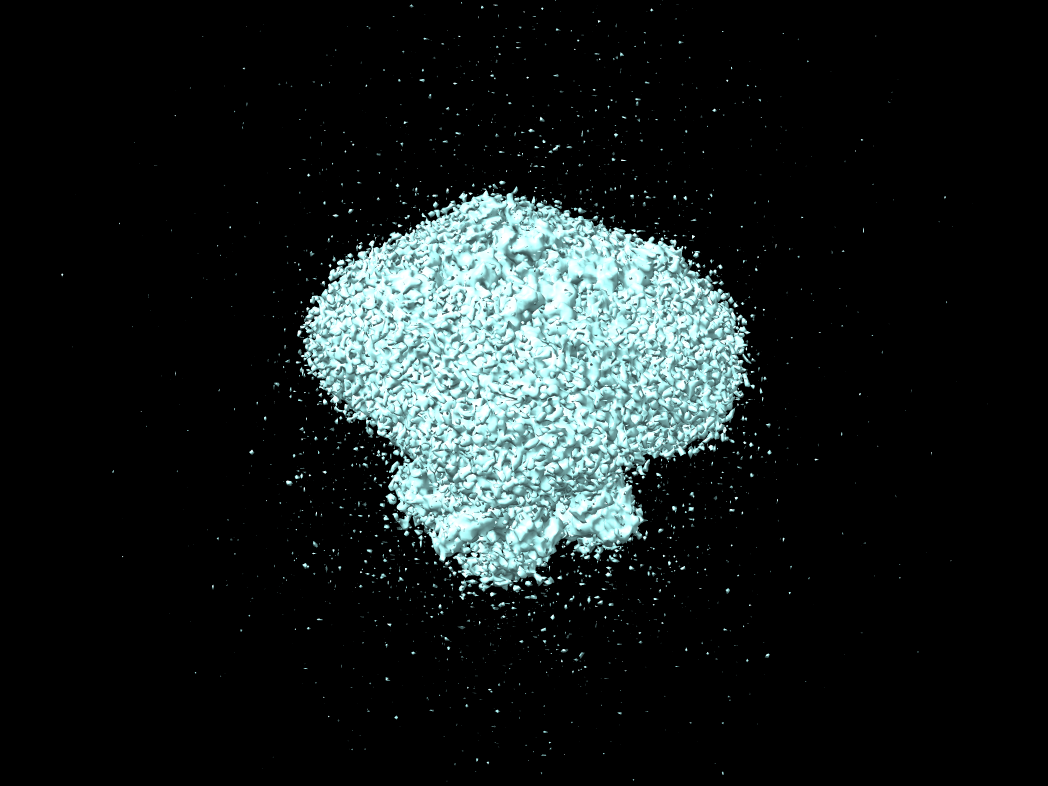}
    \subcaption{Reference model}\label{fig:10081-truth 3d}
    \endminipage\hfill
    \caption{Visualizations of the ab initio model, the refined model, and the reference model. }
\end{figure}

\begin{figure}[!htb]
    \centering
    \minipage{0.5\textwidth}
    \includegraphics[width=\linewidth]{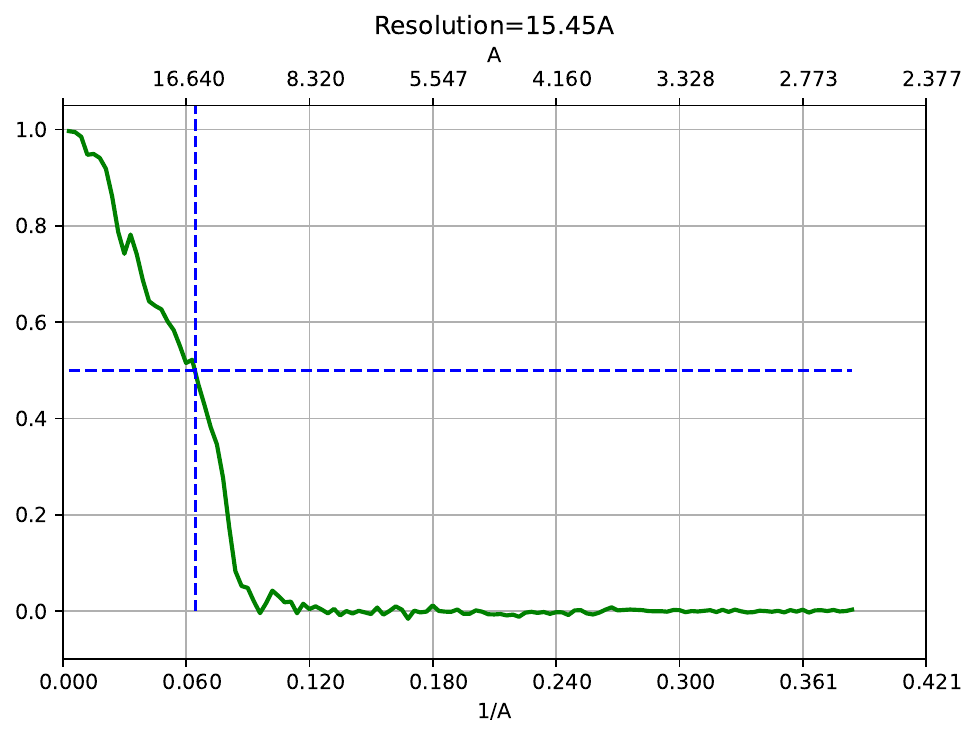}
    \subcaption{Ab initio model}\label{fig:10081-abinitio fsc}
    \endminipage
    \minipage{0.5\textwidth}
    \includegraphics[width=\linewidth]{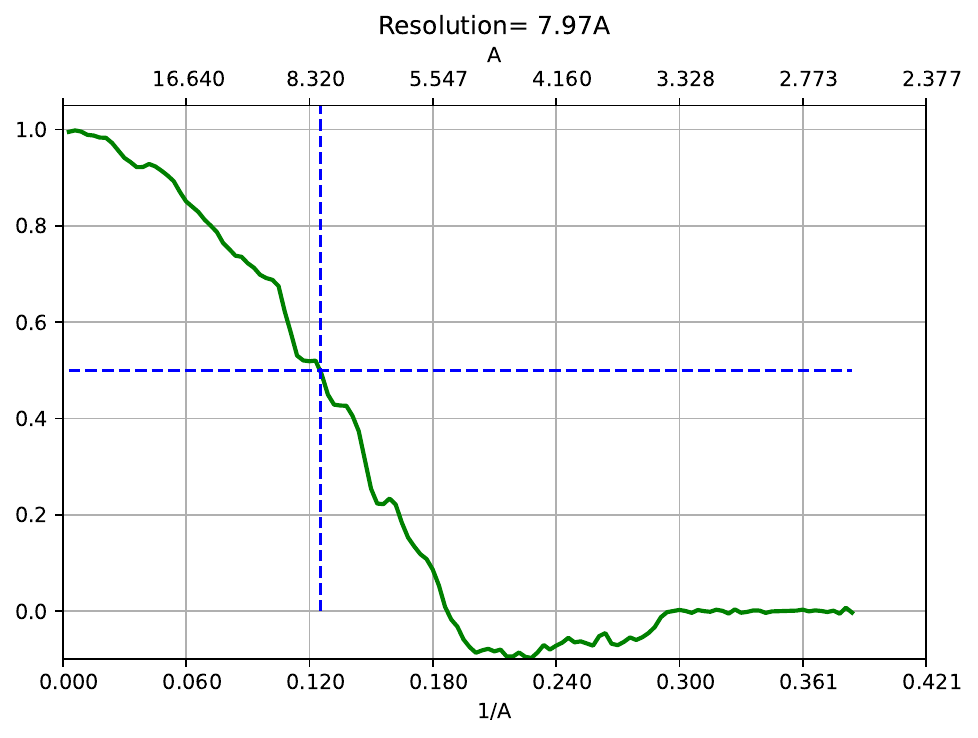}
    \subcaption{Refined model}\label{fig:10081-refine fsc}
    \endminipage
    \caption{Fourier shell correlations for the ab initio model and the refined model. }
\end{figure}
We then feed this ab initio model into CryoSPARC \citep{punjani2017cryosparc} for refinement using all the raw images and arrive at the model shown in Figure \ref{fig:10081-refine 3d} with an FSC of 7.97$\mathring{A}$, suggesting our ab initio reconstruction is reasonably accurate for iterative refinement to work.

\subsection{EMPIAR-12036}
Next we shall apply our NUG algorithm to another experimental dataset with $D_3$ symmetry taken from the \href{https://www.ebi.ac.uk/empiar/EMPIAR-12036/}{EMPIAR-12036} entry from the EMPIAR image database containing 1565602 raw images. 
We again apply RELION to generate 93 class averages of size 128 by 128  as visualized in Figure \ref{fig:12036 images} after discarding those that are visually uninterpretable.

\begin{figure}[!htb]
    \centering
    \minipage{0.2\textwidth}
    \includegraphics[width=\linewidth]{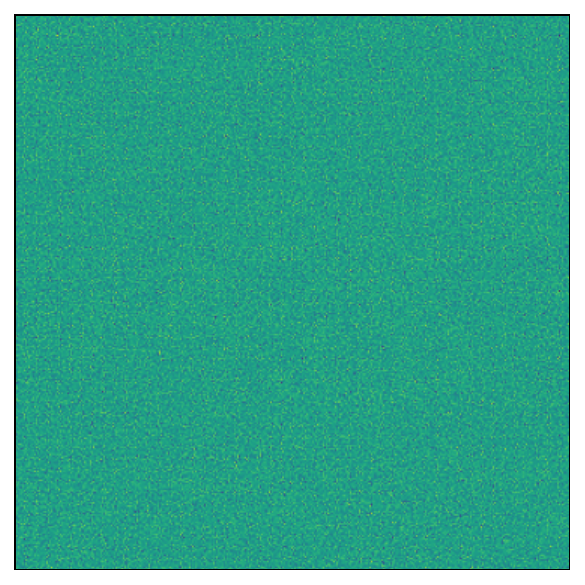}
    \endminipage\hfill
    \minipage{0.2\textwidth}
    \includegraphics[width=\linewidth]{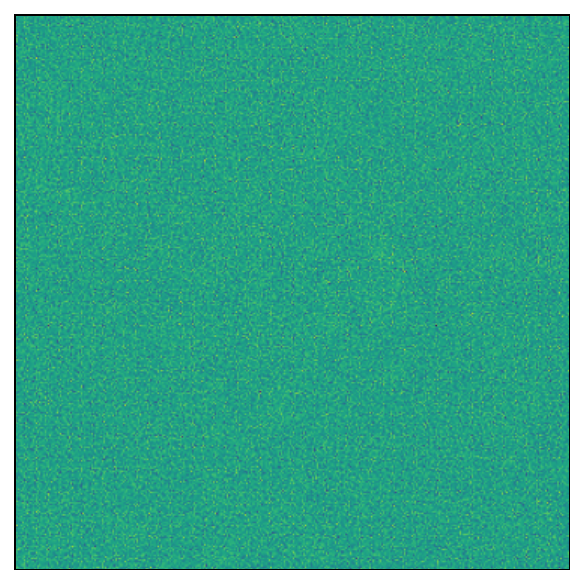}
    \endminipage\hfill
    \minipage{0.2\textwidth}
    \includegraphics[width=\linewidth]{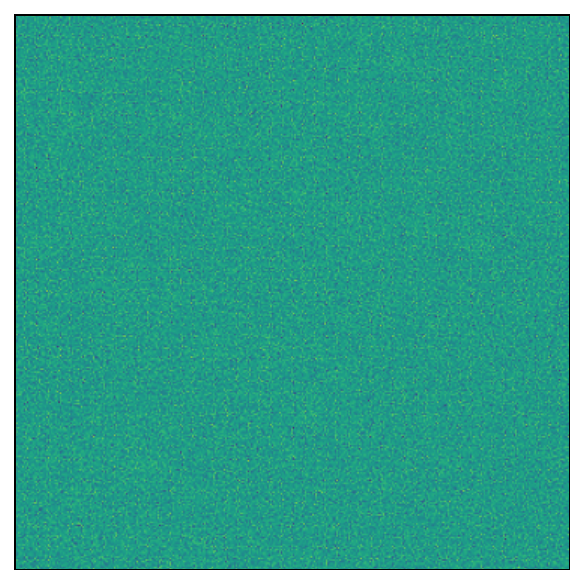}
    \endminipage\hfill
    \minipage{0.2\textwidth}
    \includegraphics[width=\linewidth]{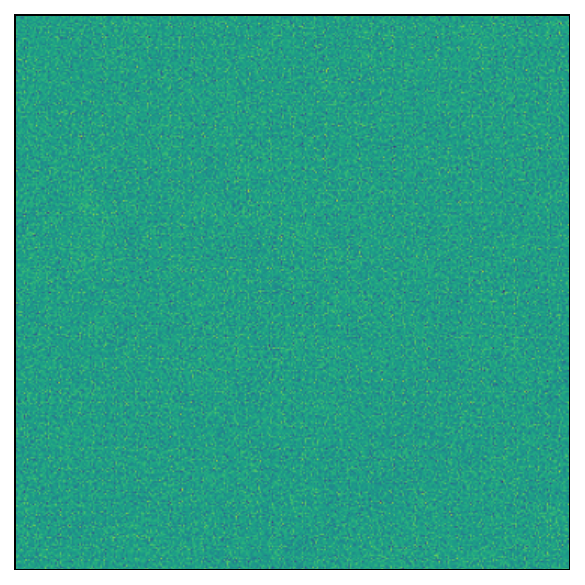}
    \endminipage\hfill
    \minipage{0.2\textwidth}
    \includegraphics[width=\linewidth]{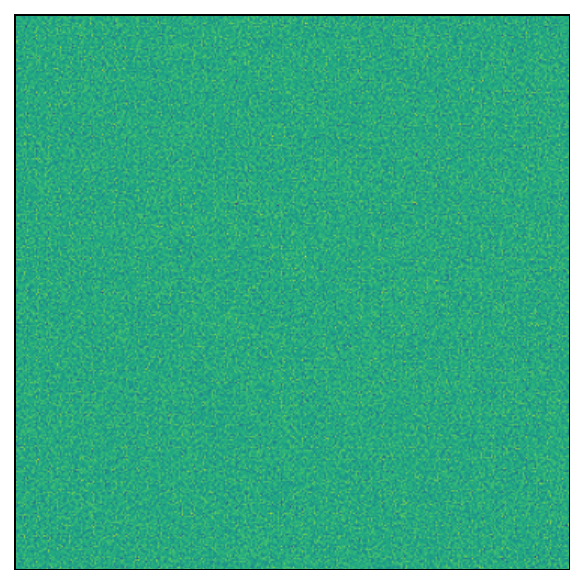}
    \endminipage\hfill

    \minipage{0.2\textwidth}
    \includegraphics[width=\linewidth]{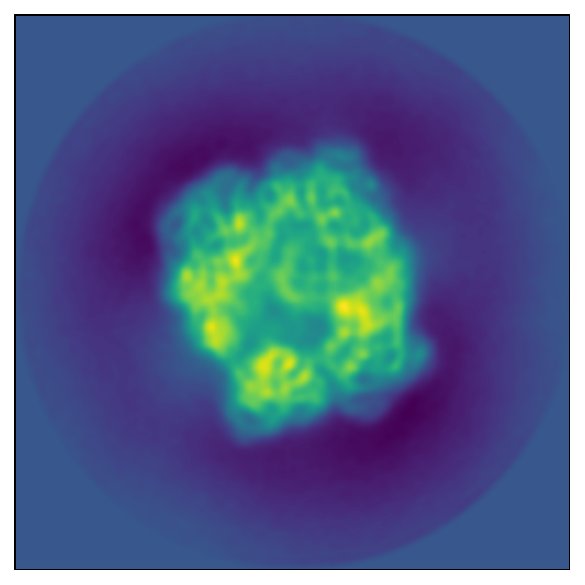}
    \endminipage\hfill
    \minipage{0.2\textwidth}
    \includegraphics[width=\linewidth]{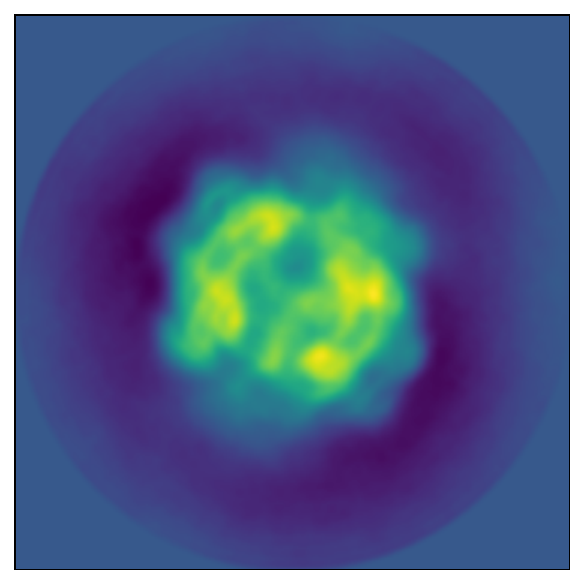}
    \endminipage\hfill
    \minipage{0.2\textwidth}
    \includegraphics[width=\linewidth]{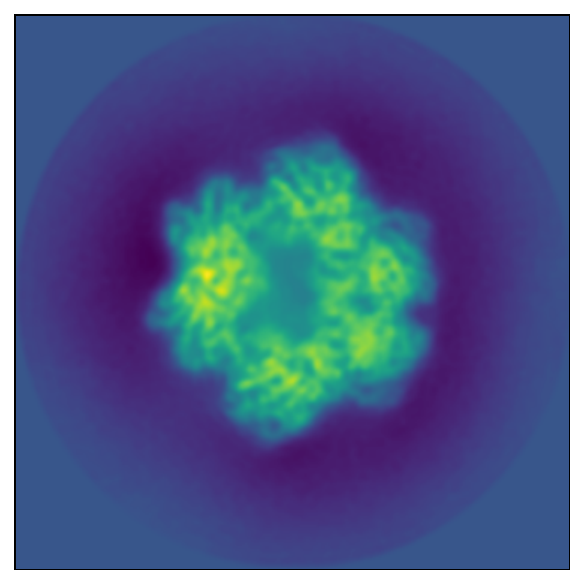}
    \endminipage\hfill
    \minipage{0.2\textwidth}
    \includegraphics[width=\linewidth]{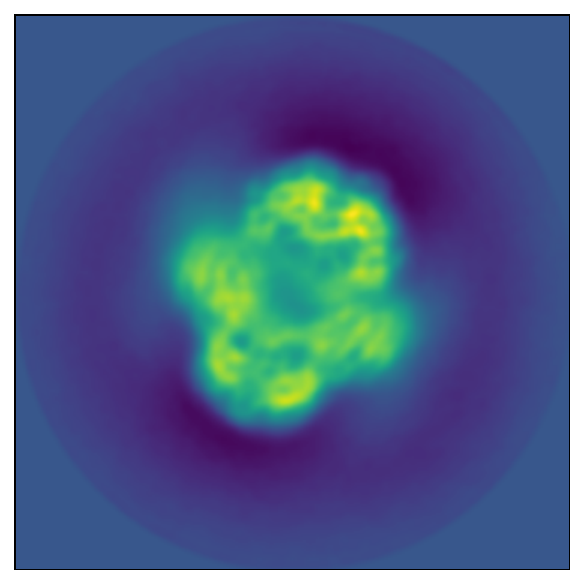}
    \endminipage\hfill
    \minipage{0.2\textwidth}
    \includegraphics[width=\linewidth]{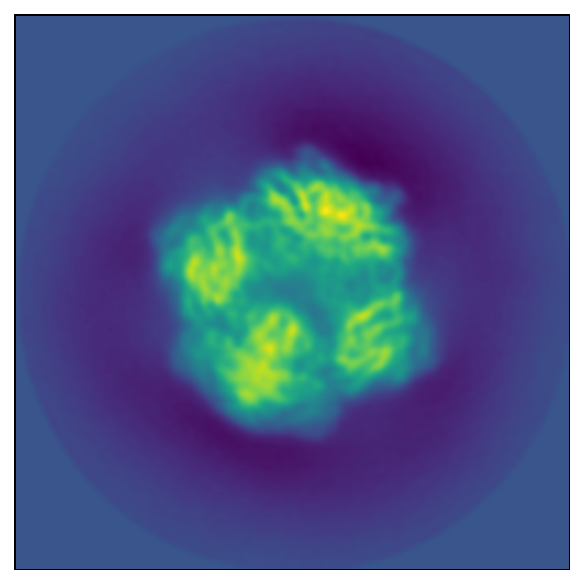}
    \endminipage\hfill
    \caption{Raw images from EMPIAR 12036 (top) and class averages computed (bottom).}\label{fig:12036 images}
\end{figure}

We apply our NUG algorithm to obtain the orientation estimations for these 93 class averages, with truncation level $K=12$ and 5 iterations of proximal refinement where in each iteration the projection rank is set to be $\operatorname{rank}(A_k)$ and the penalty is set to be $\eta_k=\frac{1}{k}$. 
The total running time is about 2 hours on an A100 GPU with 80G memory. 
The estimated orientations are used to build an ab initio model as displayed in Figure \ref{fig:12036-abinitio 3d}, yielding
a reconstruction with a Fourier shell correlation (FSC) 26.60$\mathring{A}$ with respect to the volume reported under the 12036 entry as is shown in Figure \ref{fig:12036-abinitio fsc}.
This ab initio model is then imported into CryoSPARC \citep{punjani2017cryosparc} for refinement using all the raw images and we obtain the refined model as shown in Figure \ref{fig:12036-refine 3d} with an FSC of 10.72$\mathring{A}$. 
The experiment as well as that in the previous subsection suggests the applicability of our approach to experimental datasets.

\begin{figure}[!htb]
    \centering
    \minipage{0.33\textwidth}
    \includegraphics[width=\linewidth]{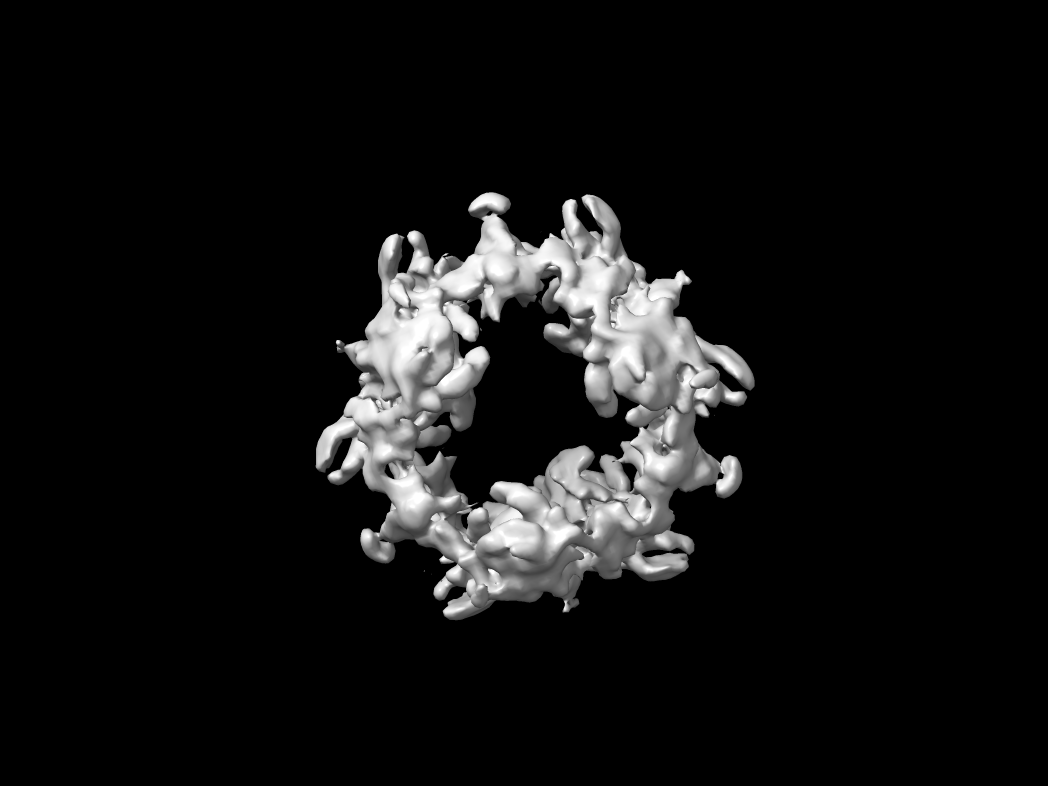}
    \subcaption{Ab initio model}\label{fig:12036-abinitio 3d}
    \endminipage\hfill
    \minipage{0.33\textwidth}
    \includegraphics[width=\linewidth]{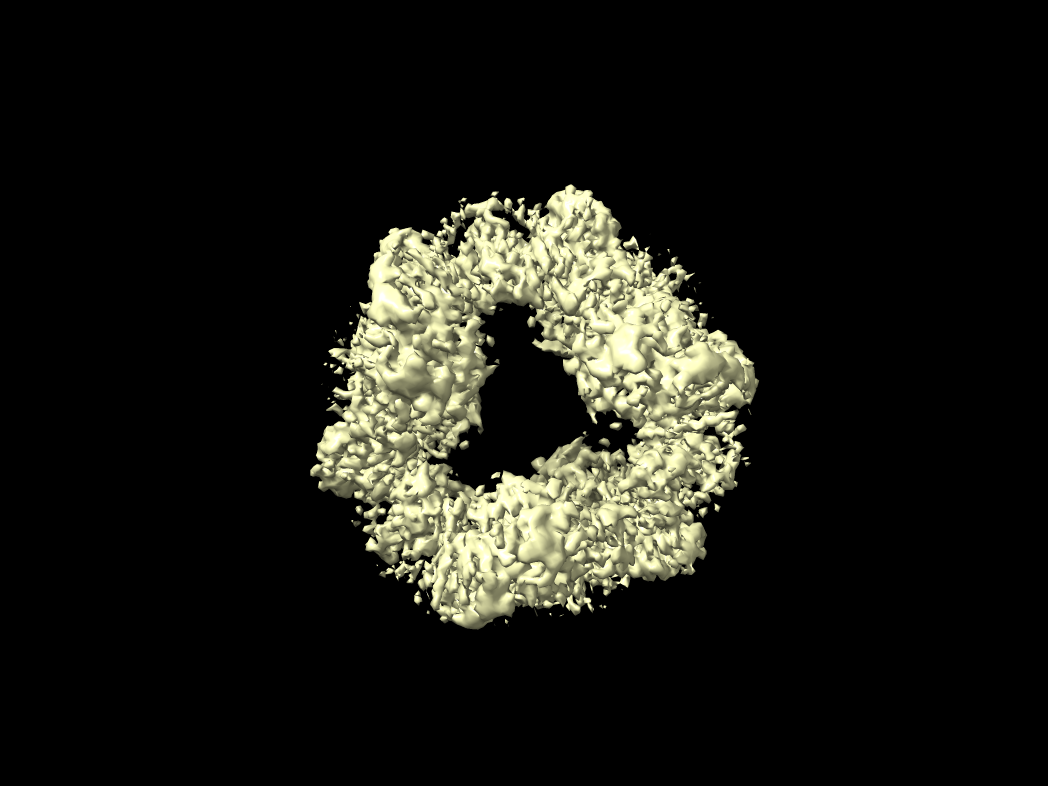}
    \subcaption{Refined model}\label{fig:12036-refine 3d}
    \endminipage\hfill
    \minipage{0.33\textwidth}
    \includegraphics[width=\linewidth]{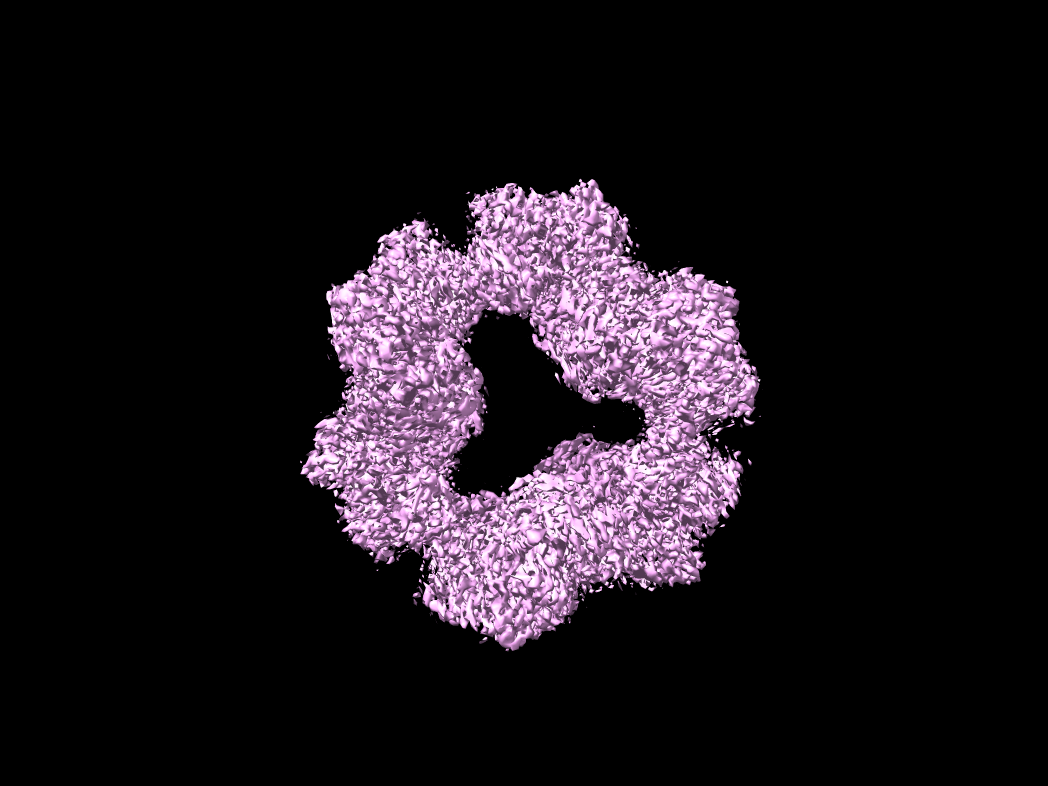}
    \subcaption{Reference model}\label{fig:12036-truth 3d}
    \endminipage\hfill
    \caption{Visualizations of the ab initio model, the refined model, and the reference model. }
\end{figure}

\begin{figure}[!htb]
    \centering
    \minipage{0.5\textwidth}
    \includegraphics[width=\linewidth]{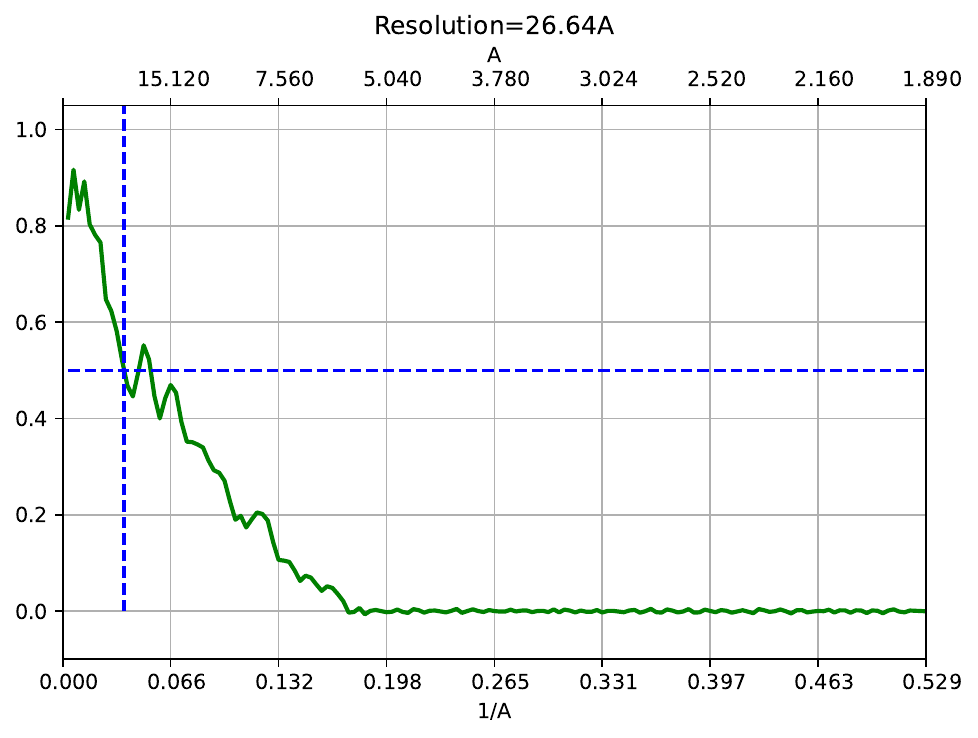}
    \subcaption{Ab initio model}\label{fig:12036-abinitio fsc}
    \endminipage
    \minipage{0.5\textwidth}
    \includegraphics[width=\linewidth]{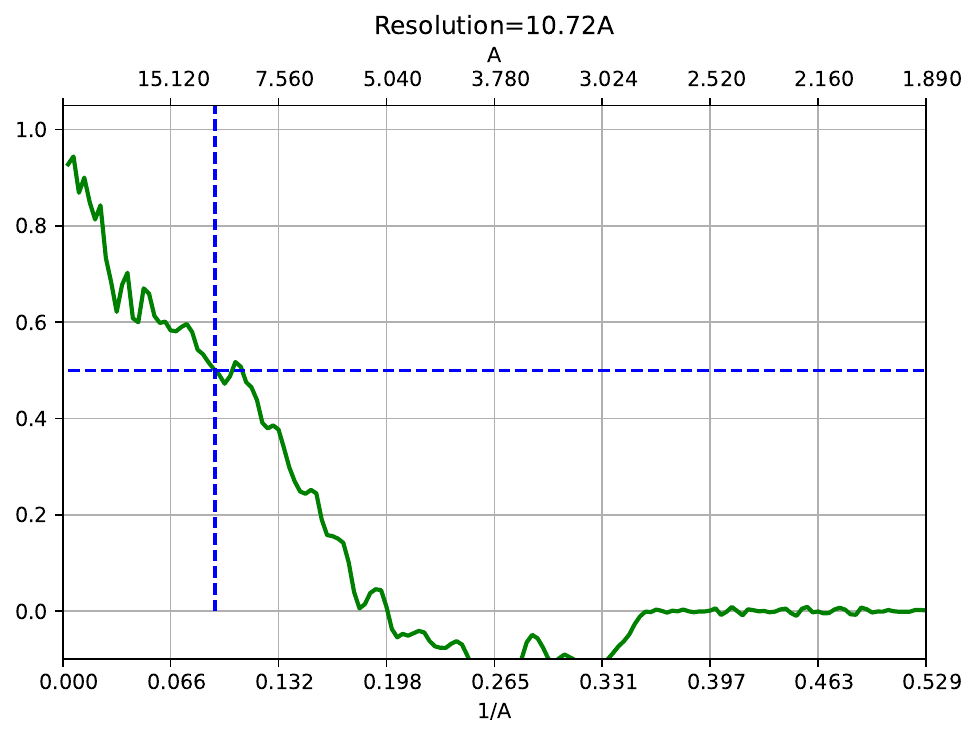}
    \subcaption{Refined model}\label{fig:12036-refine fsc}
    \endminipage
    \caption{Fourier shell correlations for the ab initio model and the refined model. }
\end{figure}

\section*{Acknowledgments}
A.S. and R.Y. were supported in part by AFOSR FA9550-23-1-0249, the Simons Foundation Math+X Investigator Award, NSF DMS2009753, NSF DMS2510039, and NIH/NIGMS R01GM13678001.
Y.S. was supported by NSF DMS2514152. 
The authors would like to thank Yoel Shkolnisky for help with generating class averages and Josh Carmichael for insights on the MSE and FSC metrics.

\appendix

\section{Wigner D Matrices}\label{sec:wigner d}
Let $\alpha,\beta,\gamma$ be the Euler angles in the ZYZ-convention. Then the Wigner D-matrix takes the form \citep{Varshalovich1988}
\begin{align} \label{eq:wigner d def}
    W^{(j)}(\alpha,\beta,\gamma)= D^{(j)}(\alpha)
    w^{(j)}(\beta)
    D^{(j)}(\gamma), \qquad 
    D^{(j)}(\theta)=
    \begin{bmatrix}
        e^{\i j\theta} & & \\
        & \ddots & \\
        & & e^{-\i j \theta} 
    \end{bmatrix},
\end{align}
where $w^{(j)}(\beta)$ is the Wigner small d-matrix taking the form 
\begin{align*}
    &\quad [w^{(j)}(\beta)]_{m'm}\\
    =& [(j+m')!(j-m')!(j+m)!(j-m)!]^{\frac12}\sum_{s=s_\text{min}}^{s_{\text{max}}}\left[\frac{(-1)^{m'-m+s}\left(\cos \frac{\beta}{2}\right)^{2j+m-m'-2s}\left(\sin \frac{\beta}{2}\right)^{m'-m+2s}}{(j+m-s)!s!(m'-m +s)!(j-m'-s)!}\right], 
\end{align*}
where $s_{\textbf{min}}=\text{max}(0,m-m')$ and $s_{\textbf{max}}=\text{min}(j+m,j-m')$. 
For $j=1$, we have 
\begin{align*}
    w^{(1)}(\beta)&= 
    \begin{bmatrix}
        \frac12 (1+\cos \beta) & \frac{\sqrt{2}}{2}\sin \beta & \frac12(1-\cos \beta) \\
        -\frac{\sqrt{2}}{2}\sin \beta & \cos \beta & \frac{\sqrt{2}}{2}\sin \beta \\
        \frac12 (1-\cos \beta) & -\frac{\sqrt{2}}{2} \sin \beta & \frac12 (1+\cos \beta)
    \end{bmatrix}\\
    W^{(1)}(\alpha,\beta,\gamma)&=
    \begin{bmatrix}
        \frac12 e^{\i \alpha }e^{\i \gamma} (1+\cos \beta)& \frac{\sqrt{2}}{2}e^{\i \alpha} \sin \beta & \frac12 e^{\i \alpha} e^{-\i \gamma} (1-\cos\beta) \\
        -\frac{\sqrt{2}}{2}e^{\i \gamma}\sin \beta & \cos \beta & \frac{\sqrt{2}}{2} e^{-\i \gamma} \sin \beta \\
        \frac12 e^{-\i \alpha} e^{\i \gamma} (1-\cos \beta) & -\frac{\sqrt{2}}{2} e^{-\i \alpha} \sin \beta & \frac12 e^{-\i \alpha}e^{-\i \gamma}(1+\cos \beta)
    \end{bmatrix}
\end{align*}

Relation between Wigner-D and rotation matrices: 
Let 
\begin{align*}
    T=\begin{bmatrix}
    -\frac{\i}{\sqrt{2}} & 0 & \frac{1}{\sqrt{2}} \\
    0 & 1 & 0 \\
    -\frac{\i}{\sqrt{2}} & 0 & -\frac{1}{\sqrt{2}}
    \end{bmatrix}.
\end{align*}
Then we have 
\begin{align*}
    T^*W^{(1)}(\alpha,\beta,\gamma)T = \begin{bmatrix}
        c_\alpha c_\gamma - s_\alpha c_\beta s_\gamma & -s_\alpha s_\beta & -c_\alpha s_\gamma - s_\alpha c_\beta c_\gamma \\
        -s_\beta s_\gamma & c_\beta & -s_\beta c_\gamma \\
        c_\alpha c_\beta s_\gamma +s_\alpha c_\gamma & c_\alpha s_\beta & c_\alpha c_\beta c_\gamma - s_\alpha s_\gamma
    \end{bmatrix}=
    \begin{bmatrix}
        R_{22} & -R_{23} & -R_{21}\\
        -R_{32} & R_{33} & R_{31} \\
        -R_{12} & R_{13} & R_{11}
    \end{bmatrix},
\end{align*}
where $R$ is the rotation matrix corresponding to the ZYZ-Euler angles $(\alpha,\beta,\gamma)$. %

\section{Invariance with respect to Global Rotation}\label{sec:invariance}
To formalize the set up, let $\phi:\R^3\rightarrow \R$ be the electric potential of the molecule and define the group action $g\circ \phi = \phi(g^{-1}x)$.  
The symmetry assumption translates to $s\circ \phi = \phi$ for all $s\in S$. 
We want to show that for a global orthogonal transformation $g\in \operatorname{O}(3)$, the set $\{g_ig^{-1}\}_{i=1}^n$ gives the same $X_{ij}^{(k)}$'s as $\{g_i\}_{i=1}^n$, which is not immediately obvious from its definition. 
The idea is that when there is a global shift, the symmetry group also changes.

Let $\widetilde{\phi}=g\circ \phi$ be the transformed molecule. We have 
\begin{align*}
    (gsg^{-1})\circ \widetilde{\phi} \overset{\text{defn. of } \widetilde{\phi}}{=} (gs)\circ \phi \overset{\text{symmetry of }\phi}{=} g\circ \phi = \widetilde{\phi},
\end{align*}
so that the symmetry group for $\widetilde{\phi}$ is $gSg^{-1}$. 
Now since $g_i\circ \phi=(g_ig^{-1})\circ \widetilde{\phi}$, the associated rotations for $\widetilde{\phi}$ are $\{g_ig^{-1}\}_{i=1}^n$ and we have 
\begin{align*}
    \widetilde{X}_{ij}^{(k)} &= \frac{1}{|S|}\sum_{\widetilde{s}\in \widetilde{S}}\rho_k(g_ig^{-1}) \rho_k(\widetilde{s})\rho_k(g_jg^{-1})^*\\
    &  = \frac{1}{|S|}\sum_{s\in S} \rho_k(g_ig^{-1}) \rho_k(gsg^{-1})\rho_k(gg_j^{-1})
    =\frac{1}{|S|}\sum_{s\in S} \rho_k(g_i) \rho_k(s)\rho_k(g_j^{-1}) = X_{ij}^{(k)}.
\end{align*}
An implication of this observation is that when recovering the $g_i$'s, we can pick a ``nice'' coordinate system over which the $A_k$'s have nice structures, and recover the $g_i$'s as the rotation with respect to such a coordinate system.

\section{Details of Least Squares in Section \ref{sec:euler angle recovery}}\label{sec:complex matrix least squares}

\subsection{Type I}\label{sec:least squares type 1}
Suppose we want to solve the problem 
\begin{align*}
    \underset{c\in C}{\operatorname{min}}\,\, \|cA-B\|_{F},
\end{align*}
where $A,B\in \C^{p\times p}$.
Writing $c=c_r+\i c_i$, $A=A_r+\i A_i$ and $B=B_r+\i B_i$ in terms of real and imaginary parts, we have 
\begin{align*}
    f(c_r,c_i) := \|cA-B\|_F^2  = \|c_r A_r  - c_iA_i - B_r\|_F^2 + \|c_iA_r+c_rA_i -B_i\|_F^2.
\end{align*}
Setting partial derivatives to be zero, we obtain
\begin{align*}
    0=\frac{\partial f}{\partial c_r} & = 2(\langle A_r,A_r\rangle + \langle A_i,A_i\rangle)c_r - 2\langle A_r,B_r\rangle - 2\langle A_i,B_i\rangle \quad \Longrightarrow \quad c_r = \frac{\langle A_r,B_r\rangle + \langle A_i,B_i\rangle}{ \langle A_r,A_r\rangle+ \langle A_i,A_i\rangle},\\
    0=\frac{\partial f}{\partial c_i} & = 2(\langle A_r,A_r\rangle+ \langle A_i,A_i\rangle)c_i + 2\langle A_i,B_r\rangle- 2\langle A_r,B_i\rangle  \quad \Longrightarrow \quad  c_i=\frac{\langle A_r,B_i\rangle-\langle A_i,B_r\rangle}{\langle A_r,A_r\rangle+\langle A_i,A_i\rangle}.
\end{align*}

\subsection{Type II}\label{sec:least squares type 2}
Suppose we want to solve for the problem
\begin{align*}
    \underset{c\in \C}{\operatorname{min}}\,\, \|cA+\bar{c}B-C\|_{F},
\end{align*}
where $A,B\in \R^{p\times p}$ and $C\in \C^{p\times p}$. 
Writing $c=c_r+\i c_i$ and $C=C_r+\i C_i$ in terms of real and imaginary parts, we have 
\begin{align*}
    f(c_r,c_i):=\|cA+\bar{c}B-C\|_F^2=\|c_r A+ c_r B -C_r\|_F^2 + \|c_i A - c_i B -C_i\|_F^2.
\end{align*}
Setting partial derivatives to be zero, we obtain
\begin{align*}
    0=\frac{\partial f}{\partial c_r} &= 2\langle A+B,A+B\rangle c_r -2\langle A+B,C_r \rangle \quad \Longrightarrow \quad c_r=\frac{\langle A+B,C_r \rangle}{\langle A+B,A+B\rangle},\\
    0=\frac{\partial f}{\partial c_i} &= 2\langle A-B,A-B\rangle c_i -2\langle A-B,C_i\rangle \quad \Longrightarrow \quad c_i=\frac{\langle A-B,C_i \rangle}{\langle A-B,A-B\rangle} .
\end{align*}

\subsection{Type III}\label{sec:least squares type 3}
Suppose we want to solve for the problem
\begin{align*}
    \underset{c,d\in \C}{\operatorname{min}}\,\, \|cA_1+\overline{c}A_2+dB_1+\overline{d}B_2-C\|_{F},
\end{align*}
where $A_i,B_i\in \R^{p\times p}$ and $C\in \C^{p\times p}$.
Writing $c=c_r+\i c_i$, $d=d_r+\i d_i$, $C=C_r+\i C_i$ in terms of real and imaginary parts, we have 
\begin{align*}
    f(c_r,c_i,d_r,d_i):&= \|cA_1+\overline{c}A_2+dB_1+\overline{d}B_2-C\|_{F}^2\\
    & = \|c_r(A_1+A_2) + d_r(B_1+B_2) - C_r\|_F^2 + \|c_i(A_1-A_2)+d_i(B_1-B_2)-C_i\|_F^2.
\end{align*}
Setting partial derivatives equal to zero, we obtain the $2\times 2$ systems
\begin{align*}
    \begin{bmatrix}
        \langle A_1+A_2,A_1+A_2\rangle  & \langle A_1+A_2,B_1+B_2\rangle \\
        \langle A_1+A_2,B_1+B_2\rangle  & \langle B_1+B_2,B_1+B_2\rangle
    \end{bmatrix}
    \begin{bmatrix}
        c_r\\
        d_r
    \end{bmatrix}
    =
    \begin{bmatrix}
        \langle A_1+A_2,C_r\rangle \\
        \langle B_1+B_2,C_r\rangle
    \end{bmatrix},
\end{align*}
\begin{align*}
    \begin{bmatrix}
        \langle A_1-A_2,A_1-A_2\rangle  & \langle A_1-A_2,B_1-B_2\rangle \\
        \langle A_1-A_2,B_1-B_2\rangle  & \langle B_1-B_2,B_1-B_2\rangle
    \end{bmatrix}
    \begin{bmatrix}
        c_i\\
        d_i
    \end{bmatrix}
    =
    \begin{bmatrix}
        \langle A_1-A_2,C_r\rangle \\
        \langle B_1-B_2,C_r\rangle
    \end{bmatrix}
\end{align*}
so that $(c_r,c_i,d_r,d_i)$ can be solved with least squares.

\bibliographystyle{abbrvnat} 
\bibliography{bib}       

\end{document}